\documentclass[10pt]{article} % For LaTeX2e
\usepackage[accepted]{new_rlc}
\usepackage{amssymb,amsmath,amsfonts,bbm,amsthm}
\usepackage{mathtools}          % Extends amsmath, providing common math tools
\usepackage{mathrsfs}           % Enables \mathscr, which can work in cases that \mathcal does not
\usepackage{graphicx}           % For including images
\usepackage{tikz}

\usepackage{subcaption}         % Allows for the use of subfigures and subcaptions
\usetikzlibrary{automata, positioning, arrows}
\usepackage{algorithm,algorithmic}
\usepackage[space]{grffile}     % For spaces in image names
\usepackage[makeroom]{cancel}

\usepackage{url}                % For displaying urls
\usepackage{nicefrac}       % compact symbols for 1/2, etc.
\usepackage{microtype}      % microtypography
\usepackage{xcolor}         % colors
\usepackage{multirow}
\usepackage{graphicx,geometry}
\usepackage{empheq}
\usepackage{subcaption}
\newtheorem{theorem}{Theorem}

\newtheorem{corollary}{Corollary}
\newtheorem{assumption}{Assumption}

\newtheorem{proposition}{Proposition}
\newenvironment{proof_sketch}{%
\proof}{\endproof}

\makeatletter
\renewcommand*\env@matrix[1][*\c@MaxMatrixCols c]{%
  \hskip -\arraycolsep
  \let\@ifnextchar\new@ifnextchar
  \array{#1}}
\makeatother

\let\classAND\AND
\let\AND\relax
\usepackage{algorithmic}
\let\AND\classAND
\AtBeginEnvironment{algorithmic}{\let\AND\algoAND}
\title{Inverse Reinforcement Learning with Multiple Planning Horizons}

\author{
  Jiayu Yao \\
  \texttt{jiayu.yao@gladstone.ucsf.edu} \\
  Gladstone Institutes\\
  \And
  Weiwei Pan\\
  \texttt{weiweipan@g.harvard.edu} \\
  SEAS, Harvard University \\
  \And 
  Finale Doshi-Velez \\
  \texttt{finale@seas.harvard.edu} \\
  SEAS, Harvard University 
  \And
  Barbara E Engelhardt\\
  \texttt{bengelhardt@stanford.edu} \\
  Gladstone Institutes \& Stanford University \\
}

\begin{document}

\maketitle

\begin{abstract}
We study an inverse reinforcement learning (IRL) problem where the experts are planning \textit{under a shared reward function but with different, unknown planning horizons}. Without the knowledge of discount factors, the reward function has a larger feasible solution set, which makes it harder for existing IRL approaches to identify a reward function. 
To overcome this challenge, we develop two algorithms that can learn a global multi-agent reward function with agent-specific discount factors that reconstruct the expert policies. We characterize the feasible solution space of the reward function and discount factors for both algorithms and demonstrate the generalizability of the learned reward function across multiple domains. 
\end{abstract}
% To address the failure modes of naive adaptations of LP-IRL and MCE-IRL, we develop two novel algorithms to learn a global multi-agent reward function with agent-specific discount factors based on LP-IRL and MCE-IRL. We (1) characterize the feasible solution space of the reward function and discount factors for both algorithms, and (2) empirically demonstrate the generalizability of the learned reward function across multiple domains.

\section{Introduction}
\label{subsec:intro}
Designing reward functions in reinforcement learning (RL) that appropriately capture key aspects of a real-world task can be difficult in domains such as healthcare~\citep{riachi2021challenges} and finance~\citep{charpentier2021reinforcement}. Inverse reinforcement learning (IRL) addresses this challenge by learning a reward function from expert demonstrations for a given task. The learned reward serves as a succinct description of the task and then can be transferred to similar tasks.
Recent research in IRL has focused on the additional challenge of learning the reward function from heterogeneous expert behaviors,
for example, where 
expert demonstrations vary in quality~\citep{shiarlis2016inverse,brown2019extrapolating}, or where each expert is optimizing a different reward function~\citep{mendez2018lifelong,gleave2018multi,yu2019meta}.

In this work, we focus on the setting where the expert behaviors vary because each expert is optimizing for \textit{a different planning horizon, using the same reward function}. 
In RL, the planning horizon is encoded in the discount factor, which discounts future (expected) rewards attained by a given policy. % expert policy
A small discount factor corresponds to a short planning horizon, implying that the expert prioritizes short-term goals, whereas a large discount factor corresponds to a long planning horizon. 
For example, in an intensive care unit, the ventilation weaning practice is influenced by specific unit protocols and team or individual physician preferences~\citep{kapnadak2015clinical}. Here, a more aggressive weaning practice implies a small discount factor, since it prioritizes immediate changes in the patient's state. In a mobile health application that helps users manage their own wellness, users may select the planning horizon by choosing to set up short-term (< 1 year), long-term (1-2 years), or maintenance (> 2 years) health goals~\citep{dicianno2017design}.

Existing IRL work often first chooses discount factors based on domain knowledge, and then learns the reward function by fixing these discount factors~\citep{ng2000algorithms, ziebart2010modeling, ramachandran2007bayesian}. 
For entropy-regularized Markov decision processes (MDPs), when observing multiple experts and when the true discount factors of the experts are known, prior work shows that IRL can identify the true reward function up to a constant ~\citep{cao2021identifiability,rolland2022identifiability}. However, in most applications, we do not know the set of true discount factors a priori -- this set must be learned alongside the reward function.
Unfortunately, when the discount factors are misspecified or unknown, current IRL literature does not address the inference, nor the identifiability of the reward function.  
In this work, we fill this gap in the literature and study settings where both the discount factors (one per expert) and the reward function are unknown.

% JY: addressing the reviewer's concern about our motivation of assuming a global reward function (instead of agent-specific reward as well as discount factor)
In this work, we assume that a global reward function is shared among the experts as the first step in understanding how unknown discount factors pose challenges to existing IRL work.  
This setting corresponds to experts having a shared goal, but different attention to the time needed to achieve them.  
We first provide analyses on the hardness of adapting existing IRL approaches to our problem setting, where both the set of discount factors and the reward must be inferred from the data. In particular, we consider two classes of popular IRL approaches: linear programming IRL (LP-IRL;~\cite{ng2000algorithms}) and max causal entropy IRL (MCE-IRL;~\cite{ziebart2010modeling}). We show that naive extensions of LP-IRL admit undesirable feasible solutions such as the degenerate solution wherein multiple experts are assigned the same discount factor. 

In the case of MCE-IRL, whose Lagrangian dual problem can be interpreted as maximum likelihood IRL (ML-IRL)~\citep{zeng2022maximum}, we show that when discount factors are unknown, strong duality does not hold for MCE-IRL and ML-IRL. Thus, solving ML-IRL does not guarantee convergence towards a feasible solution that reconstructs the expert policies. 
% Furthermore, we show that for MCE-IRL when the discount factors are misspecified (i.e. learned incorrectly), then either: (1) there does not exist a feasible reward function that recovers the expert policy, or (2) there exists a unique feasible reward function (up to a constant), which recovers the expert policy but may not be the true reward function.  
Furthermore, we show that for MCE-IRL, when the discount factors are misspecified, then either: (1) there does not exist a reward function that recovers the expert policy, or (2) there exists a unique reward function (up to a constant), which recovers the expert policy but may not be the true reward function.  
Fortunately, in Section \ref{sec:id_and_gen}, we observe that, in practice, when there are more than two experts, the set of reward functions that recovers the expert policy is non-empty only for a small set of discount factors. That is, in many applications, both the discount factors and the reward function are identifiable. 

Finally, to address the failure modes of naive adaptations of LP-IRL and MCE-IRL, we develop two novel algorithms to learn a global multi-agent reward function with agent-specific discount factors based on LP-IRL and MCE-IRL. We (1) characterize the feasible solution space of the reward function and discount factors for both algorithms, and (2) empirically demonstrate the generalizability of the learned reward function across multiple domains.

\section{Related Work}
\textbf{IRL with homogeneous demonstrations.}
Previous IRL work focuses on identifying a reward function that explains expert behavior when the demonstrations are generated by a single expert. For example, max-margin IRL methods~\citep{ng2000algorithms,abbeel2004apprenticeship} seek a reward function that maximally separates the optimal policy and the second-most optimal policy. Max entropy IRL methods~\citep{ziebart2008maximum,ziebart2010modeling} estimate a reward function that maximizes the likelihood of the expert demonstrations. Bayesian IRL methods~\citep{ramachandran2007bayesian,jin2010gaussian} use prior knowledge to infer a posterior distribution over all possible reward functions. 
However, when given trajectories from multiple experts, each of these IRL approaches learns a \textit{separate} reward function for each expert by default, which is data inefficient. In our work, we focus on a common scenario where the \textit{global} reward function is shared by all experts with discount factors specific to each expert; these different discount factors lead to different optimal policies.

\textbf{IRL with heterogeneous demonstrations.}
Recent IRL work explores heterogeneous expert demonstrations. 
Some methods study the scenario where expert demonstrations vary in quality.
For example,~\citet{shiarlis2016inverse} learn from
demonstrations of both optimal policies and policies with undesirable behaviors (e.g., violating safety constraints). 
Similarly,~\citet{brown2019extrapolating} assume that one has access to a set of demonstrations ranked by their expected return. 
Other work studies the setting where each expert optimizes for a different task. For example, ~\citet{babes2011apprenticeship} first identify the tasks by clustering the expert demonstrations and then identify a reward function for each task. 
In~\cite{yu2019meta}, the authors use deep latent generative models to capture the shared reward structure of expert demonstrations. Finally,~\citet{mendez2018lifelong} consider a lifelong learning setting where the agent faces a sequence of similar tasks and optimizes overall performance. 
In contrast, we consider the scenario where experts share the same reward function but have different planning horizons. To the best of our knowledge, we are the first to study this type of IRL problem. 

% in related work, make it a bit more specific
\textbf{Identifiability
% \jy{generalizability} 
in IRL with respect to the reward function and discount factors.} Recent work provides identifiability analysis on the reward function when the true discount factors are known for entropy-regularized MDPs~\citep{cao2021identifiability,rolland2022identifiability}. Specifically, the authors show that, given optimal policies from two distinct discount factors and a shared reward function, one can identify the true reward function up to a constant.
However, their analysis assumes that the true discount factors are provided, which is not realistic. In this work, we show that given misspecified discount factors, the feasible reward function set may not include the true reward function under the same rank conditions in~\cite{cao2021identifiability}. 
We develop algorithms to recover the set of discount factors and the reward function for settings where both are unknown.

\section{Problem Setting}\label{sec:prob_set}
\paragraph{Markov decision processes (MDPs).}
Consider an MDP, $\mathcal{M^*} = (\mathcal{S},\mathcal{A},r^*,T,\gamma^*)$,  where $\mathcal{S}$ is a finite state space, $\mathcal{A}$ is a set of discrete actions, $T: \mathcal{S}\times\mathcal{A}\times\mathcal{S}\rightarrow[0,1]$ is the transition dynamics describing the probability of reaching the next state $s'$ by taking action $a$ in the current state $s$, $r^*(s)$ is an action-independent reward function, and $\gamma^*\in[0,1]$ is the discount factor that controls the weight of the future reward. 
% The optimal policies are defined differently under different frameworks.
For standard MDPs, the optimal policy is defined as,
\begin{equation}\label{eqn:opt_pi_standard}
    \pi^*(a|s) = \arg\max_{a} Q^{r^*,\gamma^*}_\pi(s,a)  = \arg\max_{a}
    % \left(\sum_{s'}T(s'|s,a)r^*(s')+
    \left.\mathbb{E}_\pi\left[\sum_{t=0}^\infty {\gamma^*}^{t} r^*(S_{t+1})\right\vert S_t = s, A_t = a \right].
\end{equation}
For entropy-regularized MDPs, the optimal policy is defined as,
\begin{equation}
\label{eqn:opt_pi_soft}
    \tilde{\pi}^*(a|s) = \arg\max_{\pi} \tilde{Q}^{r^*,\gamma^*}_\pi(s,a)  = \arg\max_{\pi} 
    % \left(\sum_{s'}T(s'|s,a)r^*(s')+
    \left.\mathbb{E}_\pi\left[\sum_{t=0}^\infty {\gamma^*}^{t} \left[r^*(S_{t+1})+\lambda\mathcal{H}(\pi(\cdot|S_t))\right]\right\vert S_t = s, A_t = a\right],
\end{equation}
where $\lambda$ is a temperature parameter with a larger $\lambda$ leading to a more stochastic policy. When $\lambda\rightarrow 0$, we will recover the expert policy of standard MDPs ($\pi^* = \tilde{\pi}^*$). 

\paragraph{Multi-planning horizon IRL (MP-IRL).} We assume that we are given an MDP, $\mathcal{M}^*\backslash \{r^*,\gamma^*\}$, with an unknown reward function and discount factor. We observe demonstrations from a set of $K$ expert policies, each optimized under a shared global reward function $r^*$, and transition dynamics $T$, but using distinct discount factors ($\gamma^*_i\neq \gamma^*_j$ for $i\neq j$). We denote the set of distinct discount factors by $\Gamma^*=\{\gamma^*_k\}_{k=1}^K$.
We assume that the expert policies are solved using Eq.~\ref{eqn:opt_pi_standard} or ~\ref{eqn:opt_pi_soft} given appropriate contexts, and denote the set of expert policies for standard or entropy regularized MDPs by $\Pi^*=\{\pi^*_k\}_{k=1}^K$ or $\tilde\Pi^*=\{\tilde{\pi}^*_k\}_{k=1}^K$, respectively.

% We consider a simple scenario where the expert policies can be fully observed. 
In MP-IRL,
we wish to find a reward function $r$ and a set of distinct discount factors $\Gamma=\{\gamma_k\}_{k=1}^K$ such that each expert policy 
remains optimal under the reconstructed MDP $\mathcal{M} = (\mathcal{S},\mathcal{A},r,T,\gamma_k)$. 
% We denote the experty policy by $\pi^*_k$, $\tilde{\pi}_k^*$ for standard and entropy regularized MDPs, respectively.

\section{Algorithms for MP-IRL: LP-IRL}\label{sec:lp_irl}
In this section, we introduce a popular class of IRL algorithms, linear programming IRL (LP-IRL), for the single known discount factor IRL setting. 
We explain how naive extensions of LP-IRL to the MP-IRL setting fail. Finally, we present a novel algorithm, multi-planning horizon LP-IRL (MPLP-IRL), that extends LP-IRL to 
jointly learn a set of distinct discount factors and a global reward function.

In the following, we assume that the expert policies are obtained by solving standard MDPs, and are denoted by, $\Pi^*=\{\pi^*_k\}_{k=1}^K$.
In~\citet{ng2000algorithms}, the authors assume that expert demonstrations are obtained from a single known planning horizon. They learn a reward function from expert demonstrations by solving the following optimization problem,
\begin{subequations}
 \label{eqn:lp_irl_single}
\begin{align}
    \max_{r} &\sum_{s\in S} \min_{a\in \mathcal{A}\setminus \pi^*(s)} \{Q^{r,\gamma}_{\pi^*}(s,\pi^*(s)) - Q^{r,\gamma}_{\pi^*}(s,a)\}-\lambda\|r\|_1 \label{eqn:lp_irl_obj}\\
    \text{subject to } & Q^{r,\gamma}_{\pi^*}(s,\pi^*(s)) - Q^{r,\gamma}_{\pi^*}(s,a) \geq 0 \quad \text{for } s\in\mathcal{S},a\in \mathcal{A}\setminus \pi^*(s)\label{eqn:lp_irl_constraints}\\
    & |r|\preceq r_{\max}
\end{align}
\end{subequations}
where the $l_1$ norm penalty, $\lambda\|r\|_1$, regularizes the sparsity of the reward function. 
In the above optimization problem, constraints (\ref{eqn:lp_irl_constraints}) ensure that $\pi^*$ is optimal under the inferred reward function $r$. The reward function solution to Eq.~\ref{eqn:lp_irl_obj} maximizes the sum of the differences of the $Q$-functions between the best and the next-best action over all states. 
% Note that the difference of the $Q$-function is linear with respect to the reward function $r$,
% $$Q^{r,\gamma}_{\pi^*}(s,\pi^*(s)) - Q^{r,\gamma}_{\pi^*}(s,a) = (T(\cdot|s,a^*) - T(\cdot|s,a))(I-\gamma T^{\pi^*})^{-1}r,$$
% where $T^{\pi^*}$ denotes the transition dynamics following the optimal policy $\pi^*$.
The LP-IRL problem in Eq.~\ref{eqn:lp_irl_single} can be solved with linear programming (LP)~\citep{ng2000algorithms}.
% maybe pull assumptions forward
We extend the LP-IRL approach to expert demonstrations with multiple planning horizons. As such, we make an additional assumption: 
\begin{assumption}\label{assump_1}
    For any two distinct expert policies $\pi^*_i,\pi^*_j\in \Pi^*,$ optimized under $\gamma^*_i,\gamma^*_j$, respectively, there exists a
    state $s\in\mathcal{S}$ such that $Q_{\pi_i^*}^{r^*,\gamma^*_i}(s,\pi_i^*(s)) > Q_{\pi_i^*}^{r^*,\gamma^*_i}(s,\pi_j^*(s))$. 
    % In addition,   $\gamma^*_i\neq \gamma^*_j$ for $i\neq j$.
\end{assumption}

% JY: addressing reviewer's concern not motivating assumption 1
The above assumption ensures that each expert policy is uniquely optimal in at least one of the states, 
i.e., no two expert policies are equally optimal in all states. 
Since we assume a common reward function, this assumption implies that $\gamma_i\neq\gamma_j$ for $i\neq j$. In other words, our assumption states that the discount factors lead to distinct policies. 
This assumption is necessary because, otherwise, we cannot distinguish the expert policies from the observed data and the reward function is non-identifiable. 
In Appendix~\ref{apdx_sec:toy_domain}, we also show that in practice, Assumption~\ref{assump_1} is rarely violated.

\subsection{Naive Extension of LP-IRL Fails}
We first note that naive extensions of the formulation in Eq.~\ref{eqn:lp_irl_single} to the MP-IRL setting return optimal solutions that violate Assumption~\ref{assump_1}.  The naive solution would be to simply maximize the sum of the differences of the $Q$-functions over all experts and states, i.e.,
\begin{subequations}
\label{eqn:naive_lp_irl_multi_gamma}
    \begin{align}
    \max_{\Gamma\in[0,1]^K} \max_{r} & \sum_{k\in[K]}\sum_{s\in\mathcal{S}}\min_{a\in A\setminus \pi^*_k(s)} \{Q^{r,\gamma_k}_{\pi^*_k}(s,\pi^*_k(s)) - Q^{r,\gamma_k}_{\pi^*_k}(s,a)\}-\lambda\|r\|_1 \label{eqn:naive_lp_irl_obj_multi_gamma}\\
    \text{subject to } &Q^{r,\gamma_k}_{\pi^*_k}(s,\pi^*_k(s)) - Q^{r,\gamma_k}_{\pi^*_k}(s,a)\geq 0 \quad \forall s\in\mathcal{S},\ a\in \mathcal{A}\backslash \pi^*_k(s), k\in [K]\\
    & |r|\preceq r_{\max}
\end{align}
\end{subequations}

In Appendix~\ref{sec:naive_lp}, we show that without further constraining the discount factors, there are situations where the above optimization problem assigns the same discount factors to multiple experts.
This implies that,  under the learned discount factors, some expert policies are not distinguishable from each other under any reward function, which violates Assumption~\ref{assump_1}.
The expert policies should be sufficiently different in terms of $Q$-functions under the learned reward function and discount factors. We modify this naive optimization problem to address this problem below in Section \ref{subsec:lp_irl_multiple_gammas}.

\subsection{Multi-planning horizon LP-IRL (MPLP-IRL)}\label{subsec:lp_irl_multiple_gammas}
To avoid undesirable global optimum in Appendix~\ref{sec:naive_lp}, we need to incorporate the constraints in Assumption~\ref{assump_1}. This is challenging because optimization problems that include strict inequality constraints may not have attainable optimal solutions. 
We avoid incorporating strict inequalities by first selecting states where the expert policies are distinguishable and then 
maximizing the differences of $Q$-functions only on those states.
The proposed MPLP-IRL problem is as follows:
\begin{subequations}\label{eqn:multi_gamma_cvx_irl}
\begin{align}
    \max_{\Gamma\in[0,1]^K}\max_{r} & \min_{\substack{k\in[K], (s,a)\in \Omega_k}} \{Q^{r,\gamma_k}_{\pi^*_k}(s,\pi^*_k(s)) - Q^{r,\gamma_k}_{\pi^*_k}(s,a)\}-\lambda\|r\|_1 \label{eqn:lp_irl_obj_multi_gamma}\\
    \text{subject to } &Q^{r,\gamma_k}_{\pi^*_k}(s,\pi^*_k(s)) - Q^{r,\gamma_k}_{\pi^*_k}(s,a)\geq 0 \quad \forall s\in\mathcal{S},\ a\in \mathcal{A}\backslash \pi^*_k(s), k\in [K]\label{eqn:lp_irl_obj_const1}\\
    & |r|\preceq r_{\max}, \label{eqn:lp_irl_obj_const2}
\end{align}
\end{subequations}
where $\Omega_k$ is a set of state-action tuples ensuring that there exists a feasible reward solution $r$ such that, for any $(s,a)\in\Omega_k,\ Q^{\gamma_k}_{\pi^*_k}(s,\pi^*_k(s)) - Q^{\gamma_k}_{\pi^*_k}(s,a)>0$. 
% jy add intuition here as well
$\Omega_k$ is constructed by solving another LP problem which
distinguishes the optimal action from other actions on as many states as possible.

\begin{theorem}
\label{thm:lp}
For a set of arbitrary distinct discount factors, $\Gamma$ ($\gamma_i\neq\gamma_j$ for $i\neq j$), let $\{z^*_k\}_{k=1}^K \ (z^*_k\in\mathbb{R}^{|\mathcal{S}|\times(|\mathcal{A}-1|)})$ be the optimal solution to the following LP problem,
\begin{subequations}
\label{eqn:omega_k}
\begin{align}
    \min_{r,\{z_k\}} & \sum_{k=1}^K\mathbf{1}^\intercal z_k \\
    \textup{subject to } 
    & Q^{r,\gamma_k}_{\pi^*_k}(s,\pi^*_k(s)) - Q^{r,\gamma_k}_{\pi^*_k}(s,a)+z_k(s,a)\geq 1\ \forall s\in\mathcal{S},\ a\in \mathcal{A}\backslash \pi^*_k(s), k\in [K]\\
     & Q^{r,\gamma_k}_{\pi^*_k}(s,\pi^*_k(s)) - Q^{r,\gamma_k}_{\pi^*_k}(s,a)\geq 0\ \forall s\in\mathcal{S},\ a\in \mathcal{A}\backslash \pi^*_k(s), k\in [K]\\
    & z_{k}\geq 0\  \forall k\in [K],
\end{align}
\end{subequations}
where $z_k(s,a)$ denotes the element of vector $z_k$ corresponding to the state-action tuple $(s,a)$. 
There exists a feasible reward solution $r$ that satisfies Assumption~\ref{assump_1} if, for any pair of expert policies $\pi^*_i,\ \pi^*_j$ ($i\neq j$), there exists a state $s$ such that $z^*_i(s,\pi^*_j(s))=0$.
\end{theorem}
\begin{proof_sketch} The optimization problem in Eq.~\ref{eqn:omega_k} is equivalent to 
\begin{equation*}
\max_r \sum_{k=1}^K \sum_{s\in\mathcal{S}}\sum_{a\in\{\mathcal{A}\setminus \pi^*_k(s)\}} \mathbbm{1}_{\{Q^{r,\gamma_k}_{\pi^*_k}(s,\pi^*_k(s)) - Q^{r,\gamma_k}_{\pi^*_k}(s,a)>0\}},
\end{equation*} 
which maximizes the number of state-action pairs where the $Q$-function difference is positive (where the expert policy is strictly better). $z^*_k(s,a)=1$ if the optimal solution cannot distinguish the optimal action, $\pi^*_k(s)$, from action $a$ on state $s$.
If under the optimal solution,
there is a pair of expert policies $\pi^*_i,\pi^*_j$ that cannot be distinguished on any states (i.e., $(s,\pi^*_i(s))\notin\Omega_j$ for $\forall s\in\mathcal{S}$), we fail to find a reward solution that does not violate Assumption~\ref{assump_1}, and we claim the inner optimization problem in Eq.~\ref{eqn:multi_gamma_cvx_irl} to be infeasible under $\Gamma$. Otherwise, we have that $\Omega_k=\{(s,a)|z^*_k(s,a)=0\}$. 
See full proof in Appendix~\ref{apdx:thm_lp}.
\end{proof_sketch}

Intuitively, our proposed optimization problem seeks a reward function that maximizes the minimal non-zero difference of $Q$-functions over states where 
 expert policies are distinguishable, thus encouraging expert policies to be sufficiently different and ensuring the satisfaction of Assumption~\ref{assump_1}. 

\subsection{Inference for MPLP-IRL}\label{subsec:lp_irl_bo}
While the objective function in~\ref{eqn:lp_irl_obj_multi_gamma} gives us the desired solutions, it is not convex with respect to the discount factors $\Gamma$. To solve for the global optima, we perform a bi-level optimization. Denote the lower-level objective function as 
 $$g(\Gamma,r) = \min_{{k\in[K],(s,a)\in\Omega_k}} \{Q^{r,\gamma_k}_{\pi^*_k}(s,\pi^*_k(s)) - Q^{r,\gamma_k}_{\pi^*_k}(s,a)\}-\lambda\|r\|_1.$$
We rewrite the optimization problem in Eq.~\ref{eqn:multi_gamma_cvx_irl} as,
\begin{equation}
    \max_{\Gamma\in[0,1]^K} g^*(\Gamma),
    \quad\text{where}\quad g^*(\Gamma)=\max_r g(\Gamma,r)\ \text{subject to constraints as in~\ref{eqn:lp_irl_obj_const1},~\ref{eqn:lp_irl_obj_const2}}.
\label{eqn:upper_level_obj} 
\end{equation}
Given any $\Gamma\in[0,1]^K$, the lower-level objective function $g$ can be solved analytically with LP. The upper-level optimization problem can be solved by performing a grid search over the space $[0,1]^K$. However, the computation complexity of grid search increases exponentially with respect to the total number of expert policies.
To improve computation efficiency, 
we use Bayesian optimization (BO) techniques, which are suitable for nonconvex objective functions that are expensive to evaluate. 
See full algorithm details in Appendix~\ref{apdx_sec:alg}.

\section{Algorithms for multi-planning horizon IRL: MCE-IRL}\label{sec:mce_irl}
% JY: shall we change this transition because we are not running experiments on continuous domains?
Although LP-IRL is easy to solve, we cannot apply it to domains with continuous state spaces as this will result in an infinite number of constraints. In this section, we switch our focus to MCE-IRL, which is a popular class of IRL algorithms for continuous domains. 
As is standard in this setting, we assume that the expert policies, $\tilde{\Pi}^*=\{\tilde{\pi}^*_k\}_{k=1}^K$, are solved with entropy-regularized MDPs and are optimizing a linear reward function, 
\[r_{\theta^*}(s) = {\theta^*}^\intercal\phi^{s}, \theta^*,\phi^s\in\mathbb{R}^{|S|}.\]

We first introduce notations used for MCE-IRL. For any state $s\in\mathcal{S}$ and action $a\in\mathcal{A}$, given an initial state distribution $\rho_0$,
the expected discounted state-action visitation count, and the expected discounted state visitation count  of policy $\tilde{\pi}$ are defined as,
$$\mu^\gamma_{\displaystyle{\tilde{\pi}}}(s,a)=\mathbbm{E}_{\displaystyle{\tilde{\pi}}}\left[\sum_{t=0}^\infty \gamma^t\mathbbm{1}_{\{S_t = s,A_t = a\}}\right],\ \text{and}\ \mu^\gamma_{\displaystyle{\tilde{\pi}}}(s)=\mathbbm{E}_{\displaystyle{\tilde{\pi}}}\left[\sum_{t=0}^\infty \gamma^t\mathbbm{1}_{\{S_t = s\}}\right]=\sum_a \mu^\gamma_{\displaystyle{\tilde{\pi}}}(s,a),$$
respectively.
 We further define the expected feature count of policy $\tilde{\pi}$ as 
 $$f^\gamma_{\displaystyle{\tilde{\pi}}} = \mathbb{E}_{\displaystyle{\tilde{\pi}}}\left[\sum_{t=0}^\infty \gamma^t \phi(S_t=s)\right]\in\mathbb{R}^{|S|}.$$ 
Note that, the expected feature count can also be written as $f^\gamma_{\displaystyle\tilde{\pi}}=\sum_{s}\mu^{\gamma}_{\displaystyle\tilde{\pi}}(s)\phi^{s}$.
 Assuming that the discount factor $\gamma$ is given, MCE-IRL solves the following constrained optimization problem:
 \begin{subequations}
 \label{eqn:mce_irl_single}
\begin{align}
\max_{\displaystyle{\mu_{\tilde{\pi}}^\gamma(s,a)}} &\quad  \mathcal{H}^\gamma_{\displaystyle{\tilde{\pi}}}
=\mathbb{E}_{\displaystyle{\tilde{\pi}}}\left[\sum_{t=0}^\infty -\gamma^t \log\tilde{\pi}(A_t|S_t)\right]
=\sum_{(s,a)} -\log\left(\frac{\mu^\gamma_{\displaystyle{\tilde{\pi}}}(s,a)}{\sum_a \mu^\gamma_{\displaystyle{\tilde{\pi}}}(s,a)}\right)\mu^\gamma_{\displaystyle{\tilde{\pi}}}(s,a)\label{eqn:mce_irl_single_obj}\\
\text{subject to} &\quad f_{\displaystyle{\tilde{\pi}}}^\gamma = f_{\displaystyle{\tilde{\pi}^*}}^\gamma\label{eqn:mce_irl_single_const1}\\
&\quad \sum_a \mu_{\displaystyle{\tilde{\pi}}}^\gamma(s',a) = \rho_0(s')+\gamma\sum_s\sum_a T(s'|s,a)\mu_{\displaystyle{\tilde{\pi}}}^\gamma(s,a) \quad \forall s'\in\mathcal{S}\label{eqn:mce_irl_single_const2}\\
&\quad\mu_{\displaystyle{\tilde{\pi}}}^\gamma(s,a)\geq 0 \quad \forall (s,a)\in(\mathcal{S}\times\mathcal{A}).
\label{eqn:mce_irl_single_const3}
\end{align}
 \end{subequations}
%
% concave
MCE-IRL~\citep{ziebart2010modeling} identifies a reward function by the principle of maximum causal entropy (Eq.~\ref{eqn:mce_irl_single_obj}) while matching the feature expectations between the expert and the learned policy (Eq.~\ref{eqn:mce_irl_single_const1}). The Bellman flow constraints in Eqs.~\ref{eqn:mce_irl_single_const2}-\ref{eqn:mce_irl_single_const3} ensure that $\mu_{\displaystyle\tilde{\pi}}^\gamma(s,a)$ are state-action visitation count of a valid stochastic policy $\tilde{\pi}$ with discount factor $\gamma$. 
% The optimization problem in Eq.~\ref{eqn:mce_irl_single} is concave and can be solved by the Lagrangian duality method~\citep{zhou2017infinite}.
When the reward function is linear,~\citet{zeng2022maximum} establish a strong duality between the MCE-IRL problem
 (Eq.~\ref{eqn:mce_irl_single}) and its Lagrangian dual problem, which is equivalent to the following ML-IRL problem:
\begin{subequations}
\label{eqn:ml_irl_single}
\begin{align}
\max_{\theta}  &\quad \mathcal{L}(\theta) =\mathbb{E}_{\displaystyle{\tilde{\pi}^*}}\left[\sum_{t=0}^\infty \gamma^t \log\tilde{\pi}^\theta(A_t|S_t)\right]\\
% =\sum_{(s,a)} \log\left(\frac{\mu^\gamma_{\tilde{\pi}^\theta}(s,a)}{\sum_a\mu^\gamma_{\tilde{\pi}^\theta}(s,a)}\right)\mu^\gamma_{\tilde{\pi}^*}(s,a)\label{eqn:ml_irl_single_obj}
\text{subject to} &\quad \tilde{\pi}^\theta = \arg\max_{\pi} \tilde{Q}^{\theta,\gamma}_\pi(s,a),
\end{align}
\end{subequations}
where $\mathcal{L}(\theta)$ is the expectation of the discounted likelihood of expert trajectories under policy $\tilde{\pi}^\theta$. In practice, one often solves the above ML-IRL problem instead because it is more tractable. 

\subsection{Strong duality does not hold for multi-planning horizon MCE-IRL (MPMCE-IRL)}
\label{sec:mce-irl-theory}
Although extending the formulation in Eq.~\ref{eqn:mce_irl_single} to the MP-IRL setting is straightforward, we cannot solve the Lagrangian dual or the ML-IRL problem as alternative optimization problems. In this section, we show that when the discount factors are unknown, 
strong duality does not hold between the MCE-IRL problem and its Lagrangian dual, which makes the inference less tractable.
We extend the MCE-IRL formulation to the MP-IRL setting by maximizing the sum of causal entropy while matching the expected feature counts for all expert policies:
\begin{subequations}
\small{
\label{eqn:mce_irl_multi}
\begin{align}
\max_{\Gamma\in[0,1]^K}\max_{\{\mu_{\displaystyle\tilde{\pi}_k}^{\gamma_k}(s,a)\}} &\quad  \sum_{k=1}^K \mathcal{H}^{\gamma_k}_{\displaystyle\tilde{\pi}_k}
=\sum_{k=1}^K\sum_{(s,a)} -\log\left(\frac{\mu^{\gamma_k}_{\displaystyle\tilde{\pi}_k}(s,a)}{\sum_a\mu^{\gamma_k}_{\displaystyle\tilde{\pi}_k}(s,a)}\right)\mu^{\gamma_k}_{\displaystyle\tilde{\pi}_k}(s,a)\label{eqn:mce_irl_multi_obj}\\
\text{subject to} &\quad f_{\displaystyle\tilde{\pi}_k}^{\gamma_k} = f_{\displaystyle\tilde{\pi}_k^*}^{\gamma_k}\quad \forall k\in[K]\label{eqn:mce_irl_multi_const1}\\
&\quad \sum_a \mu_{\displaystyle\tilde{\pi}_k}^{\gamma_k}(s',a) = \rho_0(s')+{\gamma_k}\sum_s\sum_a T(s'|s,a)\mu_{\displaystyle\tilde{\pi}_k}^{\gamma_k}(s,a)\quad\forall s'\in\mathcal{S},\ k\in[K]\label{eqn:mce_irl_multi_const2}\\
&\quad\mu_{\displaystyle\tilde{\pi}_k}^{\gamma_k}(s,a)\geq 0 \quad\forall (s,a)\in(\mathcal{S}\times\mathcal{A}),\ k\in[K]\label{eqn:mce_irl_multi_const3}.
\end{align}
}
\end{subequations}
The multi-planning horizon ML-IRL (MPML-IRL) formulation is as follows:
\begin{subequations}
\label{eqn:ml_irl_multi}
\begin{align}
\max_{\Gamma\in[0,1]^K}\max_{\theta}  &\quad \mathcal{L}(\theta,\Gamma) =\sum_{k=1}^K\mathbb{E}_{\displaystyle\tilde{\pi}^*_k}\left[\sum_{t=0}^\infty \gamma^t \log\tilde{\pi}^\theta_k(A_t|S_t)\right]\\
% =\sum_{k=1}^K\sum_{(s,a)} \log\left(\frac{\mu^{\gamma_k}_{\tilde{\pi}_k}(s,a)}{\sum_a\mu^{\gamma_k}_{\tilde{\pi}_k}(s,a)}\right)\mu^\gamma_{\tilde{\pi}^*_k}(s,a)\\
\text{subject to} &\quad \tilde{\pi}^\theta_k = \arg\max_{\pi} \tilde{Q}^{\theta,\gamma_k}_\pi(s,a)    \quad\forall  k\in[K]
\end{align}
\end{subequations}
We first show that  
strong duality does not hold between the MPMCE-IRL problem and its Lagrangian dual problem.
\begin{theorem}
\label{thm:strong_duality}
    Let $\mathcal{H}^*, \mathcal{G}^*$ be the optimal value of the MPMCE-IRL problem in Eq.~\ref{eqn:mce_irl_multi} and its Lagrangian dual problem, respectively. Let $\mathcal{L}^*$ be the optimal value of the MPML-IRL problem in Eq.~\ref{eqn:ml_irl_multi}.
    Then, we have that $\mathcal{G}^* \geq\mathcal{H}^*\geq-\mathcal{L}^*$. 
    % is equality ever ahieved
\end{theorem} % conditions under which strong duality holds?
\begin{proof_sketch}
The MPMCE-IRL problem in Eq.~\ref{eqn:mce_irl_multi} is nonconcave because the constraints in Eq.~\ref{eqn:mce_irl_multi_const2} are not affine. In Appendix~\ref{apdx:thm_duality}, we show that there are no saddle points for the Lagrangian dual function. Thus, strong duality does not hold. 
% proof skecth: we show an optimal solution to the MCE-IRL problem is not a critical point of the lagrangian dual
\end{proof_sketch}
Theorem~\ref{thm:strong_duality} also implies that, for MP-IRL problems, solving the MPMCE-IRL problem is not equivalent to solving the MPML-IRL problem. In Proposition~\ref{prop:ml_irl}, we further show that an optimal solution to the MPML-IRL problem may not reconstruct the expert policies. Thus, we cannot solve the MPML-IRL as an alternative even though it is more computationally convenient.
% JY: appendix?
\begin{proposition}\label{prop:ml_irl}
Let $\Gamma$, $\theta$ be an optimal solution to the MPML-IRL in Eq.~\ref{eqn:ml_irl_multi} and $\tilde{\pi}^\theta_k$ be the optimal policy for the reward parameters $\theta$ and the discount factor $\gamma_k$. Then $\tilde{Q}^{{\theta},{\gamma}_k}_{\displaystyle{\tilde{\pi}}^\theta_k}\geq \tilde{Q}^{{\theta},{\gamma}_k}_{\displaystyle{\tilde{\pi}}_k^*}$ (i.e., $\tilde{\pi}_k^*$ may not be optimal under the optimal solution $\Gamma,\theta$).
\end{proposition}
\begin{proof_sketch}
% We show that when $\tilde{\pi}_k = \tilde{\pi}^*_k$,  $\frac{\partial \mathcal{L}(\theta,\Gamma)}{\partial \hat{\gamma}_k}\geq 0$.
Let $\hat\theta$, $\hat\Gamma$ be any reward parameters and discount factors in the parameter space such that for all $k\in[K],\ \tilde{\pi}^{\hat\theta}_k = \tilde{\pi}^*_k $. 
In Appendix~\ref{apdx_sec:prop_proof}, we show that such a feasible solution that reconstructs expert policies $\tilde{\Pi}^*$ may not be a critical point of $\mathcal{L}(\theta,\Gamma)$.
\end{proof_sketch}

\subsection{Inference for multi-planning horizon MCE-IRL}
Our theory in Section~\ref{sec:mce-irl-theory} above implies that for MP-IRL problems, we cannot solve the MPMCE-IRL problem by solving the MPML-IRL problem.  To solve the MPMCE-RL problem, similar to MPLP-IRL (Sec.~\ref{subsec:lp_irl_bo}), we propose a bi-level optimization in Eq.~\ref{eqn:mce_irl_multi}. Given a fixed set of discount factors, $\Gamma\in[0,1]^K$, let the lower-level optimization problem be
\begin{equation}
\label{eqn:mce_inner}
    g^*(\Gamma) = \max_{\{\mu_{\displaystyle{\tilde{\pi}_k}}^{\gamma_k}(s,a)\}}  \sum_{k=1
    } \mathcal{H}^{\gamma_k}_{\displaystyle{\tilde{\pi}}_k} \quad \text{subject to constraints as in \ Eq.~\ref{eqn:mce_irl_multi_const1}-\ref{eqn:mce_irl_multi_const3}}.
\end{equation}

In practice,  we observe that the above optimization problem is feasible for only a small set of discount factors (see the full discussion in Sec.~\ref{sec:identifiablity}). 
To ensure convergence to a feasible solution, for each lower-level optimization problem, we calculate the duality gap between the primal problem and its Lagrangian dual problem. 
For the upper-level optimization problem, we use BO to quickly identify discount factors with feasible reward functions. 
% If the duality gap is nonzero, the optimization problem is infeasible, and we let the optimal value be $0$. 
The full algorithm is given in Algorithm~\ref{alg:mceirl}.
\subsection{Feasibilty and Identifiability Analysis for Inference}\label{sec:identifiablity}
Although the optimization problem in Eq.~\ref{eqn:mce_inner} is concave and can be solved with the Lagrangian duality method, it may not be feasible when the discount factors are misspecified.
% discuss feasibility here
In this section, we study the effect of misspecified or unknown discount factors on the identifiability of the true reward function, which further determines the feasible solution space of the MPMCE-IRL problem.

Similar to~\cite{rolland2022identifiability}, we first give matrix rank conditions that determine the set of reward functions that reconstruct optimal policies, $\tilde{\Pi}^*$. 
\begin{proposition}
\label{prop:id}
    Consider an MP-IRL problem defined in Sec.~\ref{sec:prob_set} with $K$ expert policies, $\tilde\Pi^*=\{\tilde{\pi}^*_k\}_{k=1}^K$, solved with entropy-regularized RL. Assume that we are given a set of arbitrary discount factors, $\Gamma=\{\gamma_k\}_{k=1}^K$ ($\gamma_i\neq \gamma_j$ for $i\neq j$). $T_{a_i}$ is the transition dynamics under action $a_i$. Let 
    \begin{align*}
        &T_\mathcal{A} = 
        \begin{bmatrix}
            T_{a_1}\\
            \vdots\\
            T_{a_{|\mathcal{A}|}}
        \end{bmatrix},\quad
        \Phi = 
       \begin{bmatrix}
        T_{\mathcal{A}} & -(\mathbf{1}\otimes I-\gamma_1 T_{\mathcal{A}})& 0  & \cdots\\
        % T_{|\mathcal{A}|} & 0 & -(\mathbf{1}\otimes I-\gamma_2 T_{|\mathcal{A}|}) & \cdots\\    
        \vdots&\vdots&\ddots&\vdots \\
        T_{\mathcal{A}} & 0 & \cdots & -(\mathbf{1}\otimes I-\gamma_k T_{\mathcal{A}})\\
    \end{bmatrix}\\
\text{and}\quad
    &b^\intercal=\lambda
    \begin{bmatrix}
        \log \pi^*_1(a_1|\cdot),
        \cdots,
        \log \pi^*_1(a_{|\mathcal{A}|}|\cdot),
        \cdots,
        \log \pi^*_k(a_1|\cdot),
        \cdots,
        \log \pi^*_k(a_{|\mathcal{A}|}|\cdot)
        \end{bmatrix}^\intercal.
   \end{align*}
 Then there does not exist any reward function that reconstructs optimal policies if and only if $rank(\Phi\vert b)>rank(\Phi)$.
 There exists a unique reward function (up to a constant factor) that reconstructs optimal policies if  $rank(\Phi\vert b)=rank(\Phi)=K|\mathcal{S}|+|\mathcal{S}|-1$. 
\end{proposition} 
\begin{proof}
The proof of Proposition~\ref{prop:id}, follows in a similar manner to that of Theorem 3 in \cite{rolland2022identifiability}. To find a reward function that reconstructs optimal policies, the linear system $\Phi x=b$, with $K|\mathcal{A}||\mathcal{S}|$ equations and $(K+1)|\mathcal{S}|$ variables, needs to be consistent.  

Unlike the proof of Theorem 3 in \cite{rolland2022identifiability}, we do not assume the linear system is consistent because when the discount factors, $\Gamma$, are misspecified, the expert policies may not be optimal for the true reward function $r^*$.
According to the Rouch\'e-Capelli theorem,
the above system of equations is inconsistent if $rank(\Phi\vert b)>rank(\Phi)$. The linear system has a unique  solution (up to a constant factor) if $rank(\Phi\vert b)=rank(\Phi)=(K+1)|\mathcal{S}|-1$.
%WP: this theorem requires unfolding and a pedagogical example (not an experiment but an analytical example).
\end{proof}

Proposition~\ref{prop:id} implies that, without correctly specified discount factors, the reward set that reconstructs optimal policies may not include the true reward function, which emphasizes the importance of 
inferring the discount factors (and learning them correctly) rather than fixing them arbitrarily. 

But is it possible to identify the true discount factors and the true reward function in practice? By the proof of Proposition~\ref{prop:id}, with more expert policies, the number of constraints of the linear system grows faster than the number of variables, which makes the set of feasible solutions of the linear system smaller.
In Section \ref{sec:id_and_gen}, we empirically show that when the number of experts is sufficiently large, the reward set is non-empty for only a small set of discount factors. Thus, in highly heterogenous settings, both the reward function and the discount factors are more identifiable.

The rank conditions in Proposition~\ref{prop:id} also determine if the optimization problem in Eq.~\ref{eqn:mce_inner} is feasible.
\begin{corollary}
\label{cor:mce_feasibility}
    Given a set of discount factors $\Gamma=\{\gamma_k\}_{k=1}^K$, the optimization problem in Eq.~\ref{eqn:mce_inner} is feasible if and only if $rank(\Phi\vert b)=rank(\Phi)$ where $\Phi,\ b$ is defined in Proposition \ref{prop:id}. Additionally, if the optimization problem is feasible, the optimal solution is achieved at $\mu_{\displaystyle\tilde{\pi}_k}^{\gamma_k}(s,a)=\mu_{\displaystyle\tilde{\pi}^*_k}^{\gamma_k}(s,a)$.
\end{corollary}
\begin{proof_sketch}
We show that, when the optimization problem is feasible, an optimal solution of the Lagrangian multipliers is a reward function that induces the same visitation variables as the expert policies, $\{\mu_{\tilde{\pi}^*_k}^{\gamma_k}(s,a)\}$. 
Thus, for the primal problem to be feasible, it is sufficient to find a reward function 
that reconstructs the expert policies. See full details in Appendix~\ref{apdx:cor_mce}.
\end{proof_sketch}

% The corollary implies that the optimization problem in Eq.~\ref{eqn:mce_inner} is feasible if and only if, under a given set $\Gamma$, there exists a reward function for which the expert policies $\tilde{\Pi}^*$ are optimal. 

\section{Experiments and Results}
In this section, we provide details of our designed domains and experiment setup.  We first study the identifiability and generalizability of both the reward function and discount factors of each domain.
We then study the properties of the learned reward function and the set of discount factors of MPLP-IRL and MPMCE-IRL algorithms, and investigate how well the learned reward functions generalize to similar tasks. Last,  we demonstrate the fast convergence of our algorithms. 
\subsection{Domains}
\label{subsec:domains}
We test the MPLP-IRL (Algorithm~\ref{alg:lp}) and MPMCE-IRL (Algorithm ~\ref{alg:mceirl}) on three domains: (1) the toy domain (Fig.~\ref{fig:mdp_true_r}), in which the experts trade off between the probability of getting the reward and the reward magnitude; (2) the big-small domain~\citep{ankile2023discovering}, in which experts choose between a small reward close by or a large one that is far away; and (3) the cliff domain, in which the experts trade off the risk of falling off the cliff with a large reward.
 For each domain, we provide expert demonstrations from $3$ distinct expert policies.
 % In the toy domain, experts trade-off between the probability of getting the reward and the reward magnitude. In the big-small domain, experts choose between a small reward close by or a large one that is far away. 
See full details in Appendix~\ref{apdx_sec:domains}.

\subsection{Identifiability and Generalizability Analysis}\label{sec:id_and_gen}
For each domain, we solve the linear system $\Phi x=b$ (Proposition~\ref{prop:id}) by performing a grid search over the space of $\Gamma\in[0,1]^K$ with an interval of $0.01$. Across all designed domains, we observe that there exist reward functions that reconstruct expert policies only when $\Gamma=\Gamma^*$, which implies that the MPMCE-IRL problem (Eq.~\ref{eqn:mce_irl_multi}) has a small feasible solution space. 

For the toy domain, if we only provide expert demonstrations from $2$ expert policies, for all the given discount factors obtained form $[0,1]^2$, there always exists a unique reward function (up to a constant) that reconstructs the optimal policy, which may not include the true reward function. We further conduct a generalizability analysis of these reward functions with the full procedure described in Appendix~\ref{sec:gen_error}. In Appendix Fig.~\ref{fig:id_and_gen}, we see that $\sim 40\%$ of these feasible reward functions do not generalize well to new tasks with generalization errors (Eq.~\ref{eqn:eval_gen}) larger than $0.5$, which emphasizes the importance of learning the discount factors correctly.
\subsection{Results}
\label{sec:results}
\vspace{-0.2cm}
\begin{figure}[htbp]
\centering
\makebox[0.95\textwidth]{
 \begin{subfigure}[t]{0.45\linewidth}
     % \centering
     \includegraphics[width=0.8\textwidth]{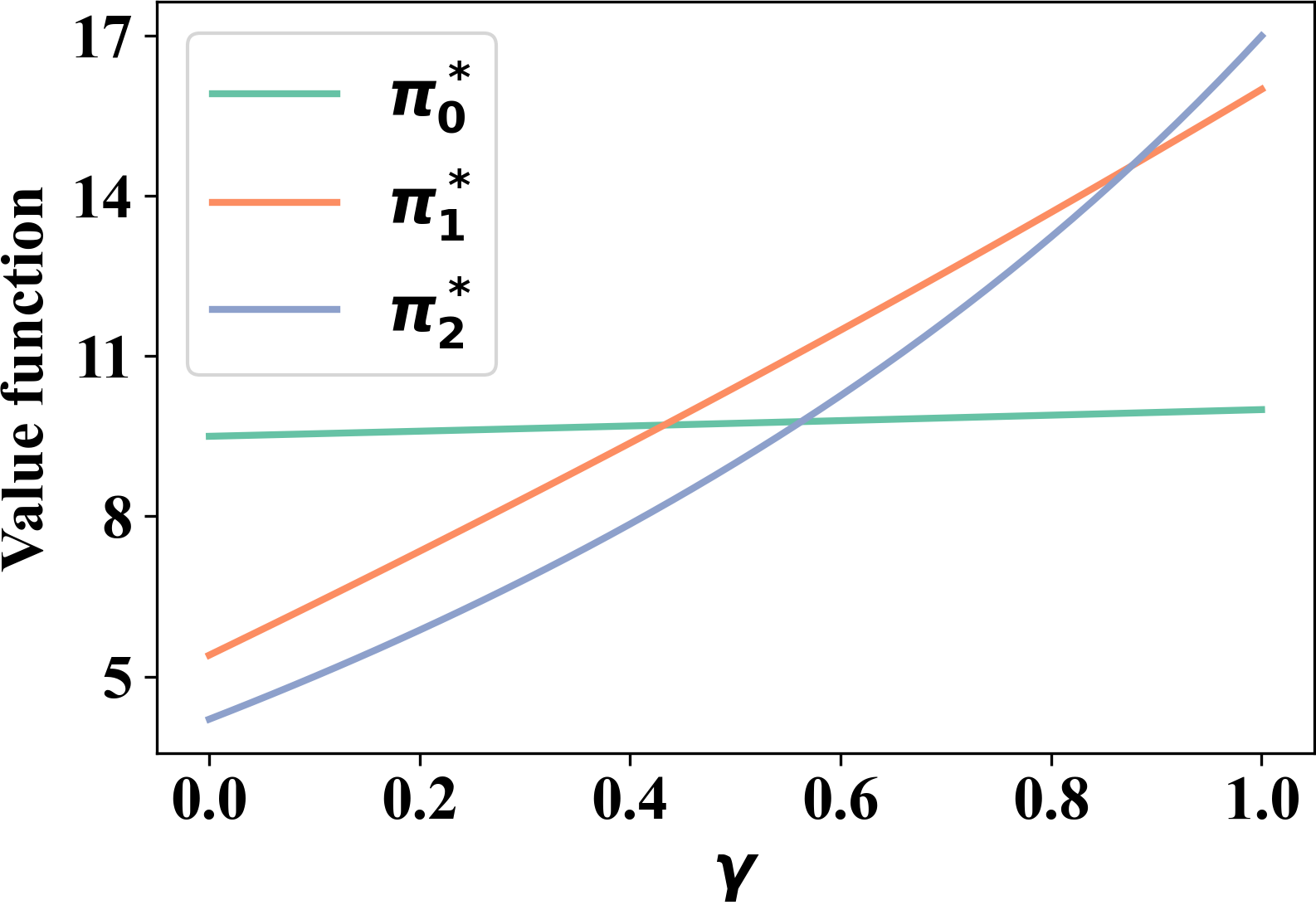}
     \caption{The value function $V^{\gamma,r*}_{\pi^*_k}(s_0)$ of expert policies $\pi^*_0,\ \pi^*_1,\ \pi^*_2$ under the true reward $r^*$.}
     \label{fig:toy_opt_policy}
 \end{subfigure}
 ~
  \begin{subfigure}[t]{0.5\linewidth}
     % \centering
     \includegraphics[width=0.7\textwidth]{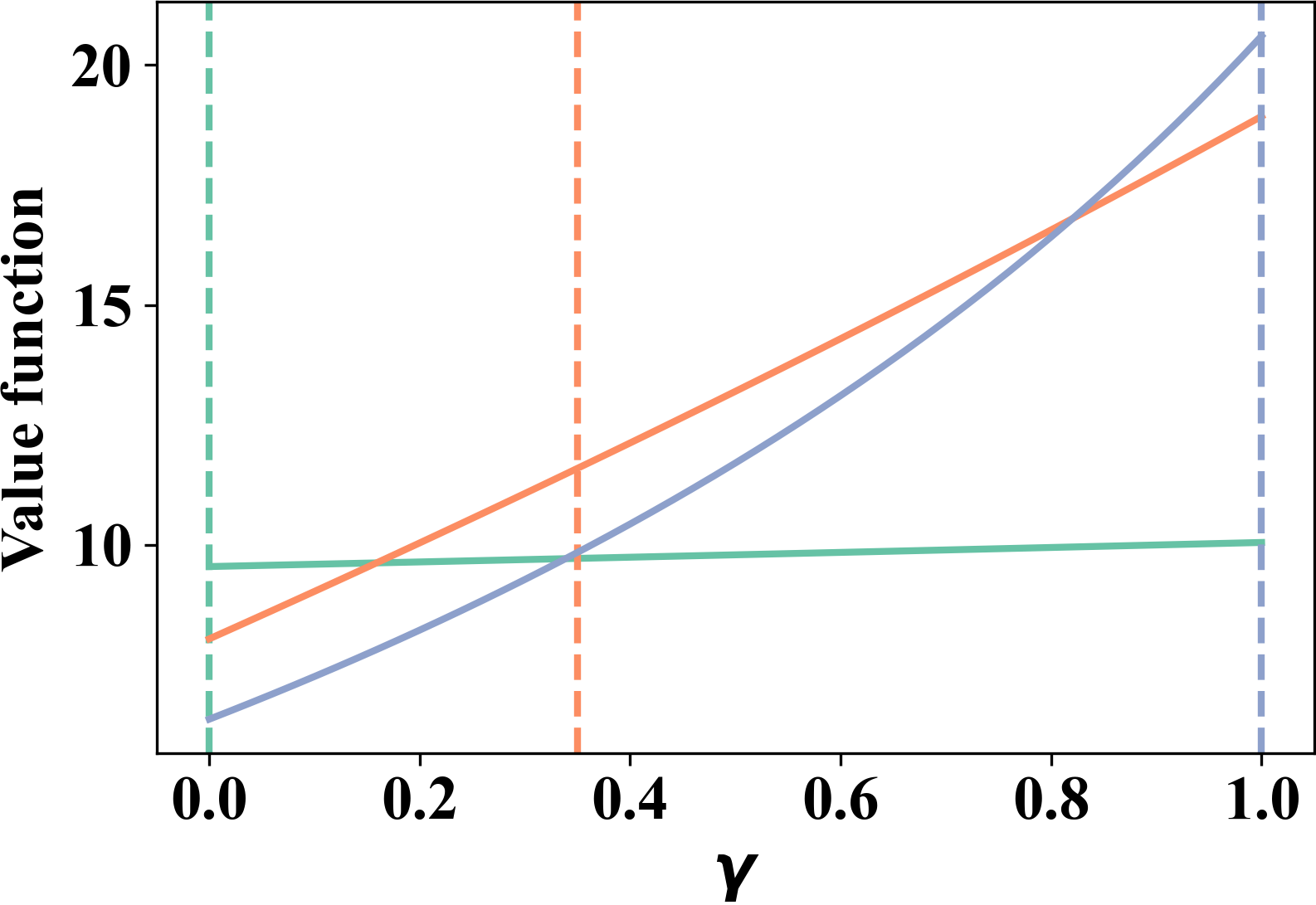}
     \raisebox{2.2em}{
     \includegraphics[width=0.2\textwidth]{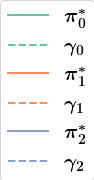}}
     \caption{The value function $V^{\displaystyle{\tilde{r}},\displaystyle{\tilde\gamma}}_{\pi_k}(s_0)$ of reconstructed optimal policies $\pi_0,\ \pi_1,\ \pi_2$ 
     under the learned reward function of MPLP-IRL, $\tilde{r}$.
     }
     \label{fig:toy_lp_policy}
 \end{subfigure}
 }
 \caption{Plots of the value function of the initial state under (a) the true reward function $r^*$, (b) the learned reward function of MPLP-IRL, $\tilde{r}$: $x,\ y$-axes represent the discount factor $\gamma\in[0,1]$ and the value function of expert policies or reconstructed optimal policies, respectively. Each color represents a different policy. The dashed lines in (b) represent the learned discount factors, $\tilde{\Gamma}$. We see that MPLP-IRL recovers the order of true discount factors.}
 \label{fig:toy_domain}
\end{figure}
\vspace{-0.2cm}
\begin{figure}[htbp]
\centering
\begin{tabular}{|l|l|l|l|}
\hline
 & Toy & Big Small & Cliff  \\ \hline
MPLP-IRL & $0.009\pm 0.040$ & $0.039\pm 0.117$ & $0.0\pm 0.0$ \\ \hline
MPMCE-IRL & $0.21\pm 0.25$ & $0.040\pm 0.065$ & $0.0\pm 0.0$ \\ \hline
\end{tabular}
\caption{The table of the generalization error (Eq.~\ref{eqn:eval_gen}) with one standard deviation of the learned reward function: each row and column represents a different algorithm and domain, respectively.}
\label{table:gen_err}
\end{figure}
\textbf{The learned discount factors recover the order of the true discount factors.} 
From Fig.~\ref{fig:toy_domain}, we see that although the learned discount factor ($\tilde{\Gamma}\approx\{0,\ 0.35,\ 1\}$) does not exactly align with the ground truth, they follow the same order of the true discount factors ($\gamma^*_0\leq \gamma^*_1\leq \gamma^*_2$). This is also true for the big-small domain and the cliff domain (see full comparison of the true discount factors and the learned discount factors in Appendix Table~\ref{table:learned_gamma}).
This property allows us to interpret the bias of each expert's goal---a small discount factor implies that the expert cares more about short-term outcomes and vice versa. 

\textbf{MPLP-IRL and MPMCE-IRL can appropriately capture key aspects of the task and the learned reward functions generalize well to similar new tasks.} 
For the toy domain, we see that both MPLP-IRL and MPMCE-IRL learn a larger reward at state $s_2$ than at state $s_1$ (Appendix Fig.~\ref{fig:toy}). When the discount factor is large, this reward structure encourages the agent to collect the large reward even in the face of more stochasticity. 
For the big-small domain, the learned reward functions have a small reward for the bottom left grid and a large reward for the bottom right grid (Appendix Fig.~\ref{fig:big_small_rs}). 
For the cliff domain, the learned reward functions have large penalties for the top rows and a large reward for the upper right grid (Appendix Fig.~\ref{fig:cliff_rs}). 
Thus, all the learned reward functions have 
 similar structures to the true reward function, which allows us to transfer the reward function to new RL tasks. The setup of the generalizability analysis is described in Appendix Section~\ref{sec:gen_error}.
 In Table~\ref{table:gen_err}, we see that our learned reward functions have good generalizability  (all the generalization errors are below $0.04$ except for that of the toy domain of MPMCE-IRL, which fails to learn $r(s_0)$ correctly (Appendix Fig.~\ref{fig:mce_learned_r})).
\begin{figure}[htbp]
\centering
\begin{subfigure}[b]{0.31\linewidth}
\includegraphics[width=\textwidth]{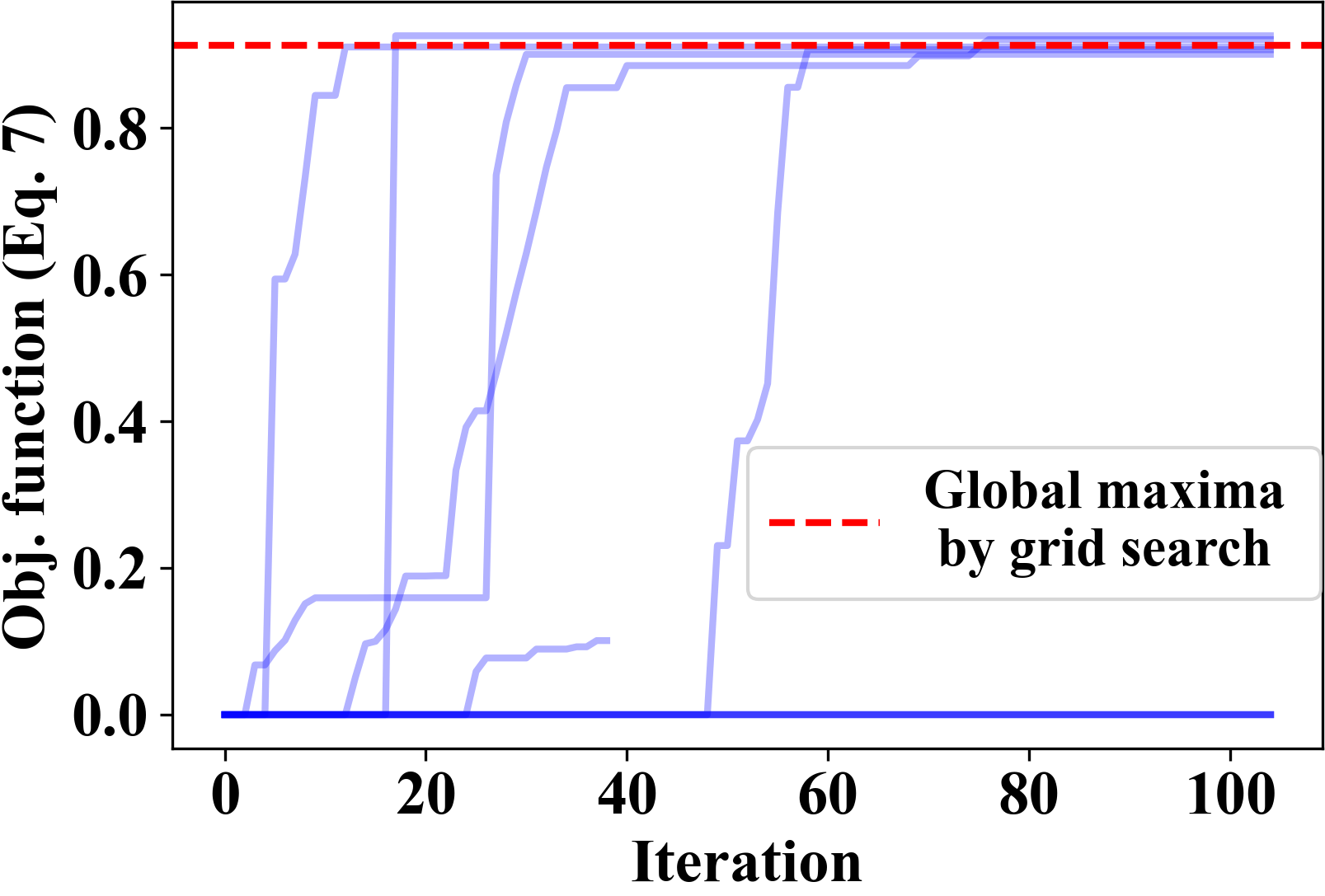}
\caption{MPLP-IRL: the toy domain}
\label{fig:toy_lp_bo_trace}
\end{subfigure}
~
\begin{subfigure}[b]{0.31\linewidth}
\includegraphics[width=0.95\textwidth]{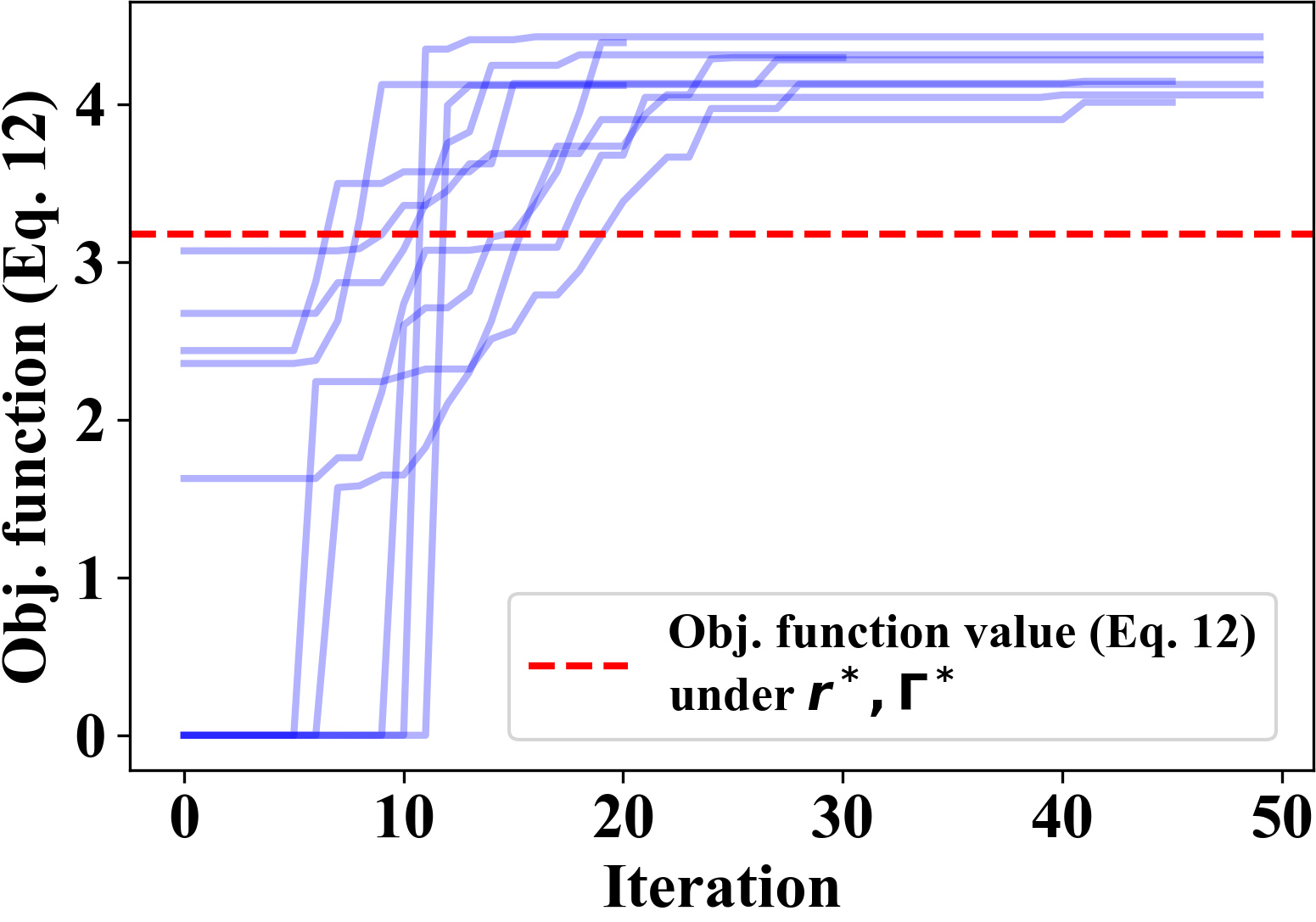}
\caption{MPMCE-IRL: the toy domain}
\label{fig:toy_mce_bo_trace}
\end{subfigure}
\caption{Trace plots of the best observed objective value of BO: $x,\ y$-axis represent the iteration and the best observed objective value, respectively. The red dashed line represents an approximate global maximum or the objective value under the ground truth.}
\label{fig:bo_trace}
\end{figure}

\textbf{MPLP-IRL and MPMCE-IRL converge quickly.} For MPLP-IRL, we approximate the global optimum by performing a grid search over $[0,1]^K$ with an interval of $0.01$. For MPMCE-IRL, it is computationally heavy to solve the optimization problem in Eq.~\ref{eqn:mce_inner} $10^6$ times. We instead compare the best current objective value of BO to the objective value evaluated under the true reward function and discount factors. 
In  Fig.~\ref{fig:bo_trace}, we see that on the toy domain, MPLP-IRL converges to the global maximum within $100$ iterations while MPMCE-IRL converges within $50$, reducing the computational burden by a factor of $\sim10^4$ compared to grid search. See Appendix~\ref{sec:result_convergence} for details.
\section{Discussion and Future Work}
In this work, we study the MP-IRL setting where each expert is planning under different planning horizons but the same reward function. We provide theoretical and empirical evidence that highlights the importance of learning correct discount factors. 

We develop two novel algorithms, MPLP-IRL and MPMCE-IRL, that learn the reward function and the discount factors jointly. 
Although MPLP-IRL is more computationally efficient (LP problems are faster to solve), it only applies to discrete domains. Additionally, MPLP-IRL does not guarantee identifying the true reward function and discount factors (in fact, for standard MDPs, identifying the set of discount factors for which the policy is optimal is nontrivial~\citep{denis2019issues}). 
In contrast, we show that when there is a sufficiently large number of experts, MPMCE-IRL can identify both the reward function and discount factors. However, in practice, MPMCE-IRL has a larger feasible solution set than Corollary~\ref{cor:mce_feasibility} suggests because we only require the algorithm to match the feature expectation in Eq.~\ref{eqn:mce_irl_multi_const1} within some threshold. Moreover, if the optimization problem in Eq.~\ref{eqn:mce_inner} is feasible for any $\Gamma\in[0,1]^K$, MPMCE-IRL does not have attainable optimal solutions.

Interesting future work includes studying when the reward function is identifiable for an MP-IRL problem and extending to an IRL setting where both planning horizons and reward functions vary.

\section{Acknowledgements}
This material is based upon work supported by the National Science Foundation under Grant No. IIS-2007076.  Any opinions, findings, and conclusions or recommendations expressed in this material are those of the author(s) and do not necessarily reflect the views of the National Science Foundation. 
JY and BEE were funded in part by Helmsley Trust grant AWD1006624, NIH NCI 5U2CCA233195, and CZI.  
BEE is a CIFAR Fellow in the Multiscale Human Program. 
\bibliography{main}
\bibliographystyle{rlc}

%%%%%%%%%%%%%%%%%%%%%%%%%%%%%%%%%%%%%%%%%%%%%%%%%%%%%%%%%%%%%%%%
%% Appendices
%%%%%%%%%%%%%%%%%%%%%%%%%%%%%%%%%%%%%%%%%%%%%%%%%%%%%%%%%%%%%%%%
\appendix
\setcounter{theorem}{0}
\setcounter{corollary}{0}
\setcounter{lemma}{0}
\setcounter{assumption}{0}
\section{Theorems}\label{apdx:proof}
\subsection{Proof of Theorem~\ref{thm:lp}}\label{apdx:thm_lp}
In this section, we prove that the pre-computed state-action tuple $\Omega_k$ allows us to find a feasible reward function that satisfies the following assumptions:
\begin{assumption}
\label{apdx:assmp}
    For any two distinct expert policies $\pi^*_i,\pi^*_j\in \Pi^*,\ i\neq j,$ optimized under $\gamma^*_i,\gamma^*_j$  ($\gamma^*_i\neq \gamma^*_j$), respectively, there is at least one
    state $s\in\mathcal{S}$ such that $Q_{\pi_i^*}^{r^*,\gamma^*_i}(s,\pi_i^*(s)) > Q_{\pi_i^*}^{r^*,\gamma^*_i}(s,\pi_j^*(s))$.
\end{assumption}
\begin{theorem}
(Restated) For a set of arbitrary distinct discount factors, $\Gamma$ ($\gamma_i\neq\gamma_j$ for $i\neq j$), let $\{z^*_k\}_{k=1}^K \ (z^*_k\in\mathbb{R}^{|\mathcal{S}|\times(|\mathcal{A}-1|)})$ be the optimal solution to the following problem,
\begin{subequations}
\label{apdx_eqn:omega_k}
\begin{align}
    \min_{r,z_k} & \sum_{k=1}^K\mathbf{1}^\intercal z_k \\
    \textup{subject to } 
    & Q^{r,\gamma_k}_{\pi^*_k}(s,\pi^*_k(s)) - Q^{r,\gamma_k}_{\pi^*_k}(s,a)+z_k(s,a)\geq 1\ \forall s\in\mathcal{S},\ a\in \mathcal{A}\backslash \pi^*_k(s), k\in [K]\\
     & Q^{r,\gamma_k}_{\pi^*_k}(s,\pi^*_k(s)) - Q^{r,\gamma_k}_{\pi^*_k}(s,a)\geq 0\ \forall s\in\mathcal{S},\ a\in \mathcal{A}\backslash \pi^*_k(s), k\in [K]\\
    & z_{k}\geq 0\  \forall k\in [K],
\end{align}
\end{subequations}
where $z_k(s,a)$ denotes the element of vector $z_k$ corresponding to the state-action tuple $(s,a)$. 
There exists a feasible reward solution $r$ that satisfies Assumption~\ref{apdx:assmp} if, for any pair of policies $\pi^*_i,\ \pi^*_j$ ($i\neq j$), there exists a state $s$ such that $z^*_i(s,\pi^*_j(s))=0$.
\end{theorem}
\begin{proof}
Given an optimization problem in the following: 
\begin{subequations}
\label{eqn:apdx_multi_gamma_irl}
\begin{align}
    \max_{r} & \sum_{(s,a)} \mathbbm{1}_
    {\{Q^{r,\gamma_k}_{\pi^*_k}(s,\pi^*_k(s)) - Q^{r,\gamma_k}_{\pi^*_k}(s,a)\}\geq 0}\\
    \text{subject to } &Q^{r,\gamma_k}_{\pi^*_k}(s,\pi^*_k(s)) - Q^{r,\gamma_k}_{\pi^*_k}(s,a)\geq 0 \quad \forall s\in\mathcal{S},\ a\in \mathcal{A}\backslash \pi^*_k(s), k\in [K]
\end{align}
\end{subequations}
Let $\hat{r}$ be an optimal solution to the optimization problem in Eq.~\ref{eqn:apdx_multi_gamma_irl}. It is easy to see that if, for any pair of policies $\pi^*_i,\ \pi^*_j$ ($i\neq j$), there exists a state $s$ such that $Q^{\hat{r},\gamma_i}_{\pi^*_i}(s,\pi^*_i(s)) - Q^{\hat{r},\gamma_i}_{\pi^*_i}(s,\pi^*_j(s)) > 0$, then Assumption~\ref{apdx:assmp} is satisfied. 

We start the proof by showing that an optimal solution to optimization problem~\ref{apdx_eqn:omega_k} is also an optimal solution to optimization problem~\ref{eqn:apdx_multi_gamma_irl}. 

With Bellman Equations, the difference of the $Q$-functions can be written as,
$$Q^{r,\gamma_k}_{\pi^*_k}(s,\pi^*_k(s)) - Q^{r,\gamma_k}_{\pi^*_k}(s,a)= (T(\cdot|s,\pi^*_k(s)) - T(\cdot|s,a))(I-\gamma_k T^{\pi^*_k})^{-1}r.$$

Let $W\in\mathbb{R}^{(|\mathcal{S}|\times (|\mathcal{A}|-1)\times K)\times |\mathcal{S}|}$ be a matrix where each row $$w^\intercal_{s,a,k}= (T(\cdot|s,\pi^*_k(s)) - T(\cdot|s,a))(I-\gamma_k T^{\pi^*_k})^{-1},\ \forall s\in\mathcal{S},a\in \mathcal{A}\setminus \pi^*_k(s), k\in[K].$$
The optimization problem in Eq.~\ref{apdx_eqn:omega_k} then can be rewritten as
\begin{subequations}
\begin{align}
\max_r & |Wr|_0\\
\text{subject to } & Wr\geq 0
\end{align}
\label{eqn:multi_gamma_cvx_irl_step_1}
\end{subequations}
% Let $R^*$ be the optimal solution to Optimization Problem~\ref{eqn:multi_gamma_cvx_irl_step_1}. Then, 
% $$\Omega_k = \{(s,a)\in(\mathcal{S}\times\mathcal{A}\setminus\pi^*)|w^\intercal_{s,a,k}R^*\neq 0\}.$$
% \end{proof}

% In the following, we show that Optimization Problem~\ref{eqn:multi_gamma_cvx_irl_step_1} can be solved with Linear Programming.
% \begin{claim}
% The optimization problem in Eq.~\ref{eqn:multi_gamma_cvx_irl_step_1} is equivalent to the following
The optimization problem in Eq.~\ref{eqn:apdx_multi_gamma_irl}  can be rewritten as
\begin{subequations}
\begin{align}
    \min_{r,z} & \mathbf{1}^\intercal z \label{eqn:step1_obj}\\
    \text{subject to } & Wr+z\geq 1\label{eqn:slack}\\
    & Wr \geq 0 \label{eqn:step1_const1}\\
    & z\geq 0\label{eqn:step1_const2}
\end{align}
\label{eqn:multi_gamma_cvx_irl_step_1_dual}
\end{subequations}

% Specifically, an optimal solution $R^*$ of Optimization Problem~\ref{eqn:multi_gamma_cvx_irl_step_1_dual} is also an optimal solution to Optimization Problem~\ref{eqn:multi_gamma_cvx_irl_step_1}. Furthermore, $\mathbf{1}^\intercal z^*=|\mathcal{S}|\times (|\mathcal{A}|-1)\times K-|WR^*|_0$
% \end{claim}
% \begin{proof}
Let $r$ be a reward function that satisfies Constraint~\ref{eqn:step1_const1}. Then for any constant $c>0$, $cr$ also satisfies Constraint~\ref{eqn:step1_const1}. Now, let $c$ be a constant such that any positive element in $cWr$ is larger than $1$. 
Denote the $i$-th element of any vector $x$ by $x_i$. An optimal solution, $z^*$ and $\hat{r}$. of optimization problem in~\ref{eqn:multi_gamma_cvx_irl_step_1_dual} has the following form:
\begin{equation}\label{eqn:opt_z}
    z^*_i = 1 \text{ if } (cW\hat{r})_i=0,\ z^*_i=0 \text{ if } (cW\hat{r})_i\geq 1.
\end{equation}

We show that $\hat{r}$ is also an optimal solution to optimization problem~\ref{eqn:multi_gamma_cvx_irl_step_1} with proof by contradiction.

Let $z^*,\hat{r}$  be an optimal solution to optimization problem~\ref{eqn:multi_gamma_cvx_irl_step_1_dual},
but $\hat{r}$ is not an optimal solution to optimization problem~\ref{eqn:multi_gamma_cvx_irl_step_1}. 
Then there exists another $\tilde{r}$ such that $W\tilde{r}\geq 0$ and $|W\tilde{r}|_0 >|W\hat{r}|_0$. 
Now, let $c$ be a constant such that any positive element in $W(c\tilde{r})$ is larger than or equal $1$ and construct a vector $\tilde{z}$ according to Eq.~\ref{eqn:opt_z}.
Because $|cW\tilde{r}|_0 =|W\tilde{r}|_0>|W\hat{r}|_0$, $W\tilde{r}\geq 0$ and $W\hat{r}\geq 0$, $\tilde{z}$ will have more zero elements than $z^*$. 
Thus, $\mathbf{1}^\intercal \tilde{z} < \mathbf{1}^\intercal z^*$, which contradicts the assumption that $z^*$ is optimal.

% With the definition of $z^*$ in Equation~\ref{eqn:opt_z}, it is easy to see that $\mathbf{1}^\intercal z^*=|z|_0=|\mathcal{S}|\times (|\mathcal{A}|-1)\times K-|Wr^*|_0$.
Additionally, with the definition of $z^*$ in Eq.~\ref{eqn:opt_z}, we can see that  $$w^\intercal_{s,a,k} = Q^{\hat{r},\gamma_k}_{\pi^*_k}(s,\pi^*_k(s)) - Q^{\hat{r},\gamma_k}_{\pi^*_k}(s,a) > 0$$ if and only if $z^*_k(s,a)=0$.
Thus, there exists a feasible reward solution $r$ that satisfies Assumption~\ref{apdx:assmp} if, for any pair of policies $\pi^*_i,\ \pi^*_j$, there exists a state $s$ such that $z^*_i(s,\pi^*_j(s))=0$.
\end{proof}

\subsection{Proof of Theorem~\ref{thm:strong_duality}}\label{apdx:thm_duality}

In this section, we show that under our targetted IRL setting in~\ref{sec:prob_set}, strong duality does not hold between the MCE-IRL problem and its Lagrangian dual. Furthermore, solving the MCE-IRL problem is not equivalent to solving the ML-IRL problem.

\begin{theorem} (Restated)
Let the multi-planning horizon MCE-IRL problem be
\begin{subequations}
\small{
\label{apdx_eqn:mce_irl_multi}
\begin{align}
\max_{\Gamma\in[0,1]^K}\max_{\{\mu_{\tilde{\pi}_k}^{\gamma_k}(s,a)\}} &\quad  \sum_{k=1}^K \mathcal{H}^{\gamma_k}_{\tilde{\pi}_k}
=\sum_{k=1}^K\sum_{(s,a)} -\log\left(\frac{\mu^{\gamma_k}_{\tilde{\pi}_k}(s,a)}{\sum_a\mu^{\gamma_k}_{\tilde{\pi}_k}(s,a)}\right)\mu^{\gamma_k}_{\tilde{\pi}_k}(s,a)\label{apdx_eqn:mce_irl_multi_obj}\\
\text{subject to} &\quad f_{\tilde{\pi}_k}^{\gamma_k} = f_{\tilde{\pi}_k^*}^{\gamma_k}\quad \forall k\in[K]\label{apdx_eqn:mce_irl_multi_const1}\\
&\quad \sum_a\mu_{\tilde{\pi}_k}^{\gamma_k}(s',a) = \rho_0(s')+{\gamma_k}\sum_s\sum_a T(s'|s,a)\mu_{\tilde{\pi}_k}^{\gamma_k}(s,a)\quad\forall s'\in\mathcal{S},\ k\in[K]\label{apdx_eqn:mce_irl_multi_const2}\\
&\quad\mu_{\tilde{\pi}_k}^{\gamma_k}(s,a)\geq 0 \quad\forall (s,a)\in(\mathcal{S}\times\mathcal{A}),\ k\in[K]\label{apdx_eqn:mce_irl_multi_const3}.
\end{align}
}
\end{subequations}
Let the multi-planning horizon ML-IRL problem be
\begin{subequations}
\label{apdx_eqn:ml_irl_multi}
\begin{align}
\max_{\Gamma\in[0,1]^K}\max_{\theta}  \quad \mathcal{L}(\theta,\Gamma) &=\sum_{k=1}^K\mathbb{E}_{\tilde{\pi}^*_k}\left[\sum_{t=0}^\infty \gamma^t \log\tilde{\pi}^\theta_k(A_t|S_t)\right]\label{apdx_eqn:ml_irl_multi_obj}\\
&=\sum_{k=1}^K\sum_{(s,a)} \log\left(\frac{\mu^{\gamma_k}_{\tilde{\pi}_k}(s,a)}{\sum_a\mu^{\gamma_k}_{\tilde{\pi}_k}(s,a)}\right)\mu^\gamma_{\tilde{\pi}^*_k}(s,a)\\
\textup{subject to} \quad \tilde{\pi}^\theta_k &= \arg\max_{\pi} \tilde{Q}^{\theta,\gamma_k}_\pi(s,a)    \quad\forall  k\in[K]
\end{align}
\end{subequations}
    Let $\mathcal{H}^*, \mathcal{G}^*$ be the optimal value of the multi-planning horizon MCE-IRL problem in Eq.~\ref{eqn:mce_irl_multi} and its Lagrangian dual, respectively. Let $\mathcal{L}^*$ be the optimal value of the ML-IRL problem in Eq.~\ref{eqn:ml_irl_multi}.
    Then, we have $\mathcal{G}^* \geq\mathcal{H}^*\geq-\mathcal{L}^*$. 
    % is equality ever ahieved
\end{theorem}
\begin{proof}\quad\newline
Part I -- Proof of $\mathcal{G}^*\geq \mathcal{H}^*$:

The first inequality holds $\mathcal{G}^*\geq \mathcal{H}^*$ by the weak duality. 
We further show that when the expert policies are stochastic, the optimal solutions to the primal problem are not critical points of the Lagrangian dual function.  
Thus, strong duality may not always hold.

The Lagrangian dual problem of the primal problem in Eq.~\ref{apdx_eqn:mce_irl_multi} is
% rewrite to gamma
\begin{align*}
\min_{\Theta} \max_{\Gamma,M} &\mathcal{G}(\Theta,\Gamma,M) \quad \text{,where }\\
\mathcal{G}(\Theta,\Gamma,M) &= \sum_{k=1}^K\sum_{(s,a)} -\log\left(\frac{\mu^{\gamma_k}_{\tilde{\pi}_k}(s,a)}{\sum_a\mu^{\gamma_k}_{\tilde{\pi}_k}(s,a)}\right)\mu^{\gamma_k}_{\tilde{\pi}_k}(s,a)+\theta_k^\intercal(f_{\tilde{\pi}_k}^{\gamma_k} -f_{\tilde{\pi}^*_k}^{\gamma_k})\\
&+ \sum_{s',k} x_{s',k} \left(\mu_{\tilde{\pi}_k}^{\gamma_k}(s') - \rho_0(s')-{\gamma_k}\sum_s\sum_a T(s'|s,a)\mu_{\tilde{\pi}_k}^{\gamma_k}(s,a)\right)\\
M &= \{\mu_{\tilde{\pi}_k}^{\gamma_k}(s,a)\},\Theta = \{\theta_k, x_{s,k}\},\Gamma=\{\gamma_k\}_{k=1}^K.
\end{align*}
We treat the constraints in Eq.~\ref{apdx_eqn:mce_irl_multi_const3} as implicit because the objective function in Eq.~\ref{apdx_eqn:mce_irl_multi_obj} is not defined under non-positive $\mu_{\tilde{\pi}_k}^{\gamma_k}(s,a)$. 
% We also transform $\Gamma=\{\gamma_k\}_{k=1}^K$ to unconstrained variables $\Delta=\{\text{logit}(\gamma_k)\}$ using the logit function.

For any Lagrangian multipliers $\Theta$, we find the critical points of $\mathcal{G}(\Theta,\Gamma,M)$ by setting the gradient to zero.
\begin{align}
    \frac{\partial \mathcal{G}}{\partial \mu_{\tilde{\pi}_k}^{\gamma_k}(s,a)} &= -\log\left(\frac{\mu^{\gamma_k}_{\tilde{\pi}_k}(s,a)}{\sum_a\mu^{\gamma_k}_{\tilde{\pi}_k}(s,a)}\right)+
    \theta_k^\intercal \phi^{s}+x_{s,k}
    - \gamma_k\sum_{s'} x_{s',k}T(s'|s,a)=0\nonumber\\
    \log\left(\frac{\mu^{\gamma_k}_{\tilde{\pi}_k}(s,a)}{\sum_a\mu^{\gamma_k}_{\tilde{\pi}_k}(s,a)}\right) &= 
    \theta_k^\intercal \phi^{s}+x_{s,k}
    - \gamma_k\sum_{s'} x_{s',k}T(s'|s,a)\label{apdx:grad_mu}
    % \frac{\partial \mathcal{G}}{\partial \delta} &= -\sigma(\delta_k)(1-\sigma(\delta_k))\left(\sum_{s'}x_{s',k}\left(\sum_s\sum_a T(s'|s,a)\mu_{\tilde{\pi}_k}^{\gamma_k}(s,a)\right)\right)\nonumber\\
    % &= -\gamma_k(1-\gamma_k)\left(\sum_{s'}x_{s',k}\left(\sum_s\sum_a T(s'|s,a)\mu_{\tilde{\pi}_k}^{\gamma_k}(s,a)\right)\right)=0\label{apdx:grad_gamma}
\end{align}
By Theorem 1 in ~\cite{cao2021identifiability}, Eq.~\ref{apdx:grad_mu} is satisfied by the state-action visitation count of the optimal policy under the reward parameters $\theta_k$ and discount factor $\gamma_k$ with $-x_{s,k}$ being its value function. Under our targetted IRL setting, the state-action visitation counts need to be induced by a global reward function. Thus, we have $\theta_1=\cdots=\theta_k=\theta$.

Denote the optimal policy under the reward parameters $\theta$ and discount factor $\gamma_k$ as $\tilde{\pi}^\theta_k$ and plug it into $\mathcal{G}$, the Lagrangian dual problem becomes
\begin{subequations}
\label{apdx_dqn:lagrangian_dual_2}
\begin{align}
    \min_{\theta} \max_{\Gamma} \ f(\theta,\gamma)&= \sum_{k=1}^K\sum_{(s,a)} -\log\left(\frac{\mu^{\gamma_k}_{\displaystyle\tilde{\pi}^\theta_k}(s,a)}{\sum_a\mu^{\gamma_k}_{\displaystyle\tilde{\pi}^\theta_k}(s,a)}\right)\mu^{\gamma_k}_{\displaystyle\tilde{\pi}^\theta_k}(s,a)+\theta_k^\intercal(f_{\displaystyle\tilde{\pi}^\theta_k}^{\gamma_k}-f_{\displaystyle\tilde{\pi}^*_k}^{\gamma_k} )\\
    \textup{subject to} \quad \displaystyle\tilde{\pi}^\theta_k &= \arg\max_{\pi} \displaystyle\tilde{Q}^{\theta,\gamma_k}_\pi(s,a)    \quad\forall  k\in[K]
\end{align}
\end{subequations}
Both the primal and the dual problems have an optimal solution because both of them feasible and bounded.
Strong duality holds iff there exists a saddle point for function $f(\theta, \Gamma)$. That is, there exists some $\tilde\theta, \tilde{\Gamma}$ such that 
\begin{equation*}
    \forall \theta, \Gamma\in[0,1]^K,\quad f(\tilde\theta, \Gamma)\leq f(\tilde\theta, \tilde\Gamma)\leq  f(\theta, \tilde\Gamma).
\end{equation*}
Furthermore, if strong duality holds, $\tilde\Gamma$ is a global maximum point of the primal problem and $\tilde\theta$ is a global minimum point of the Lagrangian dual. 
% is a saddle point iff
% \begin{equation*}
%     \tilde\Gamma \in \arg\max_{\Gamma} (\text{inf}_\theta\ f(\theta,\Gamma)).
% \end{equation*}
By Corollary~\ref{cor:mce_feasibility}, we know that
$\tilde\Gamma$ is a global maximum point iff there exist feasible reward functions that reconstruct the expert policies. 
Corollary~\ref{cor:mce_feasibility} also tells us that $\tilde{\theta}\in \arg\min f(\theta,\tilde\Gamma)$ are the reward parameters such that 
$\mu^{\gamma_k}_{\displaystyle\tilde{\pi}^{\tilde\theta}_k}=\mu^{\gamma_k}_{\displaystyle\tilde{\pi}^{*}_k}$. 

We now rewrite the Lagrangian dual function $f$ as follows:
\begin{equation*}
    f(\theta, \Gamma)= \sum_{k=1}^K\mathcal{H}_{\displaystyle\tilde{\pi}^\theta_k}^{\gamma_k}+\theta_k^\intercal(f_{\displaystyle\tilde{\pi}^\theta_k}^{\gamma_k}-f_{\displaystyle\tilde{\pi}^*_k}^{\gamma_k} ) =\sum_{k=1}^K \tilde{V}^{\theta,\gamma}_{\displaystyle\tilde{\pi}^\theta_k}-\tilde{V}^{\theta,\gamma}_{\displaystyle\tilde{\pi}^*_k}+\mathcal{H}_{\displaystyle\tilde{\pi}^*_k}^{\gamma_k}
\end{equation*}

When $\mu^{\gamma_k}_{\displaystyle\tilde{\pi}^{\tilde\theta}_k}=\mu^{\gamma_k}_{\displaystyle\tilde{\pi}^{*}_k}$, we have that $f(\tilde\theta,\tilde\Gamma)=\mathcal{H}_{\displaystyle\tilde{\pi}^*_k}^{\displaystyle\tilde\gamma_k}$. However, for such $\tilde\theta$, we can find another
$\gamma_k > \tilde{\gamma}_k$ 
for some $k\in[K]$ such that 
\begin{equation*}
f(\tilde\theta, \Gamma) =\sum_{k=1}^K \tilde{V}^{\displaystyle\tilde\theta,\gamma}_{\displaystyle\tilde{\pi}^{\displaystyle\tilde\theta}_k}-\tilde{V}^{\displaystyle\tilde\theta,\gamma}_{\displaystyle\tilde{\pi}^*_k}+\mathcal{H}_{\displaystyle\tilde{\pi}^*_k}^{\gamma_k}
\stackrel{(i)}{\geq} 0+\mathcal{H}_{\displaystyle\tilde{\pi}^*_k}^{\gamma_k}
\stackrel{(ii)}{>}\mathcal{H}_{\displaystyle\tilde{\pi}^*_k}^{\displaystyle\tilde\gamma_k}
\end{equation*}
where (i) follows from the fact that $\tilde{\pi}^{\displaystyle\tilde\theta}_k$ is optimal for reward parameter $\tilde\theta$ and discount fact $\gamma_k$ and (ii) follows from the fact that the entropy is monotonically increasing on $\gamma_k\in[0,1]$. Thus, strong duality does not hold.
% let \gamma \mu be the optimal, move lemma 1 above, T^\pi1 = T^\pi2

Part II -- Proof of $\mathcal{H}^*\geq -\mathcal{L}^*$:

By negating the objective function in Eq.~\ref{apdx_eqn:ml_irl_multi_obj}, we have that
\begin{align*}
\min_{\Gamma\in[0,1]^K}\min_{\theta}  \quad -\mathcal{L}(\theta,\Gamma) 
&=\sum_{k=1}^K\sum_{(s,a)} \log\left(\frac{\mu^{\gamma_k}_{\displaystyle\tilde{\pi}^\theta_k}(s,a)}{\sum_a\mu^{\gamma_k}_{\displaystyle\tilde{\pi}^\theta_k}(s,a)}\right)\mu^\gamma_{\displaystyle\tilde{\pi}^*_k}(s,a)\\
\textup{subject to} \quad \displaystyle\tilde{\pi}^\theta_k &= \arg\max_{\pi} \tilde{Q}^{\theta,\gamma_k}_\pi(s,a)    \quad\forall  k\in[K]
\end{align*}
Because an optimal solution to the multi-planing horizon MCE-IRL problem in Eq.~\ref{apdx_eqn:mce_irl_multi} is a feasible solution to the multi-planing horizon ML-IRL problem in Eq.~\ref{apdx_eqn:ml_irl_multi}. 
Thus, $\mathcal{H}^*\geq \mathcal{H} \geq-\mathcal{L}^* $.
\end{proof}
\subsection{Proof of Proposition~\ref{prop:ml_irl}}\label{apdx_sec:prop_proof}
In this section, we show that an optimal solution to the naive multi-planning horizon ML-IRL problem may not reconstruct the expert policies. Thus, we cannot use this formulation to learn a feasible reward function.
\begin{proposition}(Restated)
Let $\Gamma$, $\theta$ be an optimal solution to the ML-IRL in Eq.~\ref{eqn:ml_irl_multi} and $\tilde{\pi}^\theta_k$ be the optimal policy for the reward parameters $\theta$ and the discount factor $\gamma_k$. Then $\tilde{Q}^{{\theta},{\gamma}_k}_{\tilde{\pi}^\theta_k}\geq \tilde{Q}^{{\theta},{\gamma}_k}_{\tilde{\pi}_k^*}$ (i.e., $\tilde{\pi}_k^*$ may not be optimal under the optimal solution $\Gamma,\theta$).
\end{proposition}
\begin{proof}
Given any reward parameters $\theta$, the optimal value function and $Q$-value function satisfy,
\begin{gather}
\tilde{V}^{\theta,\gamma}_{\tilde{\pi}^\theta}( s)=\lambda\log\sum_{a}exp(\tilde{Q}^{\theta,\gamma}_{\tilde{\pi}^\theta}(s,a)/\lambda)\label{apdx_eqn:v_soft}\\
\pi^\theta(a|s) = exp\left(\frac { \tilde{Q}^{\theta,\gamma}_{\tilde{\pi}^\theta}(s,a)-\tilde{V}^{\theta,\gamma}_{\tilde{\pi}^\theta}( s))}{\lambda}\right)\label{apdx_eqn:pi_q_v}
\end{gather}
Given the above equations, we can rewrite the expectation of the discount likelihood of expert trajectories as,
\begin{align*}
    \mathcal{L}( \theta,\Gamma )&=\sum_{k=1}^K\mathbbm{E}_ {(S_t,A_t)\sim \tilde{\pi}^*_k} \left[ \sum_{t=0}^ {\infty } \gamma_k^{t}\log\tilde{\pi}^\theta_k ( A_ {t} |S_ {t} )\right] \\
    &\stackrel{(i)}{=}\frac{1}{\lambda} \sum_{k=1}^K\mathbbm{E}_ {(S_t,A_t)\sim \tilde{\pi}^*_k} \left[ \sum_{t=0}^ {\infty } \gamma_k^{t}\left(\tilde{Q}^{\theta,\gamma_k}_{\tilde{\pi}^\theta_k}(S_t,A_t)-\tilde{V}^{\theta,\gamma_k}_{\tilde{\pi}^\theta_k}( S_t))\right)\right] \\
    &=\frac{1}{\lambda} \sum_{k=1}^K\mathbbm{E}_ {(S_t,A_t)\sim \tilde{\pi}^*_k} \left[ \sum_{t=0}^ {\infty } \gamma_k^{t}\left(r(S_{t+1})+\gamma_k\tilde{V}^{\theta,\gamma_k}_{\tilde{\pi}^\theta_k}( S_{t+1})-\tilde{V}^{\theta,\gamma_k}_{\tilde{\pi}^\theta_k}( S_t)\right)\right] \\
    &=\frac{1}{\lambda} \sum_{k=1}^K\left(
     \mathbbm{E}_ {(S_t,A_t)\sim \tilde{\pi}^*_k} \left[ \sum_{t=0}^ {\infty } \gamma_k^{t}r(S_{t+1})\right]+
    \mathbbm{E}_ {(S_t,A_t)\sim \tilde{\pi}^*_k} \left[ \sum_{t=1}^ {\infty } \gamma_k^{t}\tilde{V}^{\theta,\gamma_k}_{\tilde{\pi}^\theta_k}( S_{t+1})\right]-
    \mathbbm{E}_ {(S_t,A_t)\sim \tilde{\pi}^*_k} \left[ \sum_{t=0}^ {\infty } \gamma_k^{t}\tilde{V}^{\theta,\gamma_k}_{\tilde{\pi}^\theta_k}( S_{t})\right]
    \right)\\
    &=\frac{1}{\lambda} \sum_{k=1}^K\left(
     \mathbbm{E}_ {(S_t,A_t)\sim \tilde{\pi}^*_k} \left[ \sum_{t=0}^ {\infty } \gamma_k^{t}r(S_{t+1})\right]-
    \mathbbm{E}_ {S_0\sim\rho} \left[ \sum_{t=0}^ {\infty } \gamma_k^{t}\tilde{V}^{\theta,\gamma_k}_{\tilde{\pi}^\theta_k}( S_0)\right]
    \right)\\
    &\stackrel{(ii)}{=}\sum_{k=1}^K\left(\frac{1}{\lambda} 
     \mathbbm{E}_ {(S_t,A_t)\sim \tilde{\pi}^*_k} \left[ \sum_{t=0}^ {\infty } \gamma_k^{t}r(S_{t+1})\right]-
    \mathbbm{E}_ {S_0\sim\rho} \left[ \log\sum_a exp\left(\frac{\tilde{Q}^{\theta,\gamma_k}_{\tilde{\pi}^\theta_k}( S_0,A_0)}{\lambda}\right)\right]\right),
\end{align*}
where (i) and (ii) follow from Eq.~\ref{apdx_eqn:v_soft} and ~\ref{apdx_eqn:pi_q_v}, respectively.
Let $\hat\theta$, $\hat\Gamma$ be any reward parameters and discount factors in the parameter space 
such that for all $k\in[K],\ \tilde{\pi}^{\hat\theta}_k = \tilde{\pi}^*_k $. 
We further map $\gamma_k$ to an unconstrained variable space: $\delta_k = \text{logit}(\gamma_k)$ and  $\Delta=\{\delta_k\}_{k=1}^K$. 

We now show that $ \frac{\partial\mathcal{L}}{\partial\delta_k}(\hat\delta_k)\leq 0$. 
% Thu, $\hat{\theta}, \hat{\Gamma}$ may not be the global maximum point of $\mathcal{L}(\theta, \Gamma)$. 
Calculating the gradient of $\mathcal{L}(\theta, \Gamma)$ with respect to $\delta_k$ gives us the following:
\begin{align}
    \frac{\partial\mathcal{L}}{\partial\delta_k}=&\frac{1}{\lambda} \frac{\partial}{\partial \delta_k}
     \mathbbm{E}_ {(S_t,A_t)\sim \tilde{\pi}^*_k} \left[ \sum_{t=0}^ {\infty } \frac{\partial\gamma_k^{t}}{\partial\delta_k}r(S_{t+1})\right]\nonumber\\
     &-\frac{1}{\lambda}
    \mathbbm{E}_ {S_0\sim\rho}\left[\sum_a exp\left(\frac{\tilde{Q}^{\theta,\gamma_k}_{\tilde{\pi}^\theta_k}( S_0,A_0)}{\lambda}-\log\sum_a exp\left(\frac{\tilde{Q}^{\theta,\gamma_k}_{\tilde{\pi}^\theta_k}( S_0,A_0)}{\lambda}\right)\right)\frac{\partial\tilde{Q}^{\theta,\gamma_k}_{\tilde{\pi}^\theta_k}( S_0,A_0)}{\partial\delta_k}\right] \nonumber\\
    \stackrel{(i)}{=}&\frac{1}{\lambda} \left(
     \mathbbm{E}_ {(S_t,A_t)\sim \tilde{\pi}^*_k} \left[ \sum_{t=0}^ {\infty } \frac{\partial\gamma_k^{t}}{\partial\delta_k}r(S_{t+1})\right]-
    \mathbbm{E}_ {S_0\sim\rho}\left[\sum_a \tilde{\pi}^\theta_k(A_0|S_0)\frac{\partial\tilde{Q}^{\theta,\gamma_k}_{\tilde{\pi}^\theta_k}( S_0,A_0)}{\partial\delta_k}\right]\right),\label{apdx_eqn:l_grad}
\end{align}
where (i) follows from Eq.~\ref{apdx_eqn:v_soft} and ~\ref{apdx_eqn:pi_q_v}.

The gradient of $\tilde{Q}^{\theta,\gamma_k}_{\tilde{\pi}^\theta_k}( S_0,A_0)$ can be further expanded as,
\begin{equation*}
\frac{\partial\tilde{Q}^{\theta,\gamma_k}_{\tilde{\pi}^\theta_k}( S_0,A_0)}{\partial\delta_k}=\mathbbm{E}_ {(S_t,A_t)\sim \tilde{\pi}^\theta_k} \left[ \sum_{t=0}^ {\infty } \frac{\partial\gamma_k^{t}}{\partial\delta_k}r(S_{t+1})\right]+\lambda\frac{\partial}{\partial\delta_k}\mathbbm{E}_ {(S_t,A_t)\sim \tilde{\pi}^\theta_k}\left[\sum_{t=0}^ {\infty } \gamma^t \tilde{\pi}^\theta_k(\cdot|S_t)\log\tilde{\pi}^\theta_k(\cdot|S_t)\right]. 
\end{equation*}
Plugging the gradient of $\tilde{Q}^{\theta,\gamma_k}_{\tilde{\pi}^\theta_k}( S_0,A_0)$ into Eq.~\ref{apdx_eqn:l_grad} and evaluating the gradient at $\tilde{\pi}^{\theta}_k = \tilde{\pi}^*_k $ gives us,
\begin{align*}
    \frac{\partial\mathcal{L}}{\partial\delta_k}(\hat{\delta}_k)\propto&
     \cancel{\mathbbm{E}_ {(S_t,A_t)\sim \tilde{\pi}^*_k} \left[ \sum_{t=0}^ {\infty } \frac{\partial\gamma_k^{t}}{\partial\delta_k}r(S_{t+1})\right]}-
    \cancel{\mathbbm{E}_ {(S_t,A_t)\sim \tilde{\pi}^*_k} \left[ \sum_{t=0}^ {\infty } \frac{\partial\gamma_k^{t}}{\partial\delta_k}r(S_{t+1})\right]}\\
    &-\lambda\frac{\partial}{\partial\delta_k}\mathbbm{E}_ {(S_t,A_t)\sim \tilde{\pi}^\theta_k}\left[\sum_{t=0}^ {\infty } \gamma_k^t \tilde{\pi}^\theta_k(\cdot|S_t)\log\tilde{\pi}^\theta_k(\cdot|S_t)\right]\\
    \stackrel{(i)}{=}&  -\lambda\mathbbm{E}_ {(S_t,A_t)\sim \tilde{\pi}^*_k} \left[\sum_{t=0}^ {\infty } t\hat\gamma_k^{t-1}\hat\gamma_k(1-\hat\gamma_k)\tilde{\pi}^*_k(\cdot|S_t)\log\tilde{\pi}^*_k(\cdot|S_t)\right]\\
     =&\sum_{t=1}^\infty  -\lambda(1-\hat\gamma_k)\mathbbm{E}_ {(S_t,A_t)\sim \tilde{\pi}^*_k} \left[\sum_{t'=t}^ {\infty } \hat\gamma_k^{t}\tilde{\pi}^*_k(\cdot|S_t)\log\tilde{\pi}^*_k(\cdot|S_t)\right]\geq 0,
\end{align*}
where (i) follows from the fact $\frac{\partial \gamma_k}{\partial \delta_k} = \frac{\partial \sigma(\delta_k)}{\partial \delta_k}=\sigma(\delta_k)(1-\sigma(\delta_k))=\gamma_k(1-\gamma_k)$.
We see that when $\hat\gamma_k\in(0,1)$ 
% and $\tilde{\pi}^*_k$ is not deterministic, 
$\frac{\partial\mathcal{L}}{\partial\delta_k}(\hat{\delta}_k)>0$ and $\tilde{Q}^{{\theta},{\gamma}_k}_{\tilde{\pi}_k^*}<\mathcal{L}^*$.
% By Lemma~\ref{apdx_lemma:grad_positive},
% we see that when $\hat\gamma_k\in(0,1)$ and $\tilde{\pi}^*_k$ is not deterministic, $\frac{\partial\mathcal{L}}{\partial\delta_k}(\hat{\delta}_k)>0$ and $\mathcal{L}^*> \tilde{Q}^{{\theta},{\gamma}_k}_{\tilde{\pi}_k^*}$.
\end{proof}
\subsection{Proof of Corollary~\ref{cor:mce_feasibility}}
In this section, we provide
 a characterization of the feasible solution space of the reward function and discount
factors for the lower-level optimization problem in Eq.~\ref{eqn:mce_irl_multi}. 
\label{apdx:cor_mce}
\begin{corollary} (Restated)
Given a set of discount factors $\Gamma=\{\gamma_k\}_{k=1}^K$, the following optimization problem
\begin{subequations}
    \small{
    \label{apdx_eqn:mce_irl_multi_inner}
    \begin{align}
    \max_{\{\mu_{\displaystyle\tilde{\pi}_k}^{\gamma_k}(s,a)\}} &\quad  \sum_{k=1}^K \mathcal{H}^{\gamma_k}_{\displaystyle\tilde{\pi}_k}
    =\sum_{k=1}^K\sum_{(s,a)} -\log\left(\frac{\mu^{\gamma_k}_{\displaystyle\tilde{\pi}_k}(s,a)}{\sum_{a}\mu^{\gamma_k}_{\displaystyle\tilde{\pi}_k}(s,a)}\right)\mu^{\gamma_k}_{\displaystyle\tilde{\pi}_k}(s,a)\label{apdx_eqn:mce_irl_multi_inner_obj}\\
    \text{subject to} &\quad f_{\displaystyle\tilde{\pi}_k}^{\gamma_k} = f_{\displaystyle\tilde{\pi}_k^*}^{\gamma_k}\quad \forall k\in[K]\label{apdx_eqn:mce_irl_multi_inner_const1}\\
    &\quad \sum_a\mu_{\displaystyle\tilde{\pi}_k}^{\gamma_k}(s',a) = \rho_0(s')+{\gamma_k}\sum_s\sum_a T(s'|s,a)\mu_{\displaystyle\tilde{\pi}_k}^{\gamma_k}(s,a)\quad\forall s'\in\mathcal{S},\ k\in[K]\\
    &\quad\mu_{\displaystyle\tilde{\pi}_k}^{\gamma_k}(s,a)\geq 0 \quad\forall (s,a)\in(\mathcal{S}\times\mathcal{A}),\ k\in[K],
    \end{align}
    }
    \end{subequations}
    is feasible if and only if $rank(\Phi\vert b)=rank(\Phi)$ where $\Phi,\ b$ is defined in Proposition \ref{prop:id}. Additionally, if the optimization problem is feasible, the optimal solution is achieved at $\mu_{\displaystyle\tilde{\pi}_k}^{\gamma_k}(s,a)=\mu_{\displaystyle\tilde{\pi}^*_k}^{\gamma_k}(s,a)$.
\end{corollary}
\begin{proof}
In~\cite{zeng2022maximum}, the authors show that when the optimization problem in Eq.~\ref{apdx_eqn:mce_irl_multi_inner} is feasible, strong duality holds.
% in corollary $\mu_{\displaystyle{\pi}_k^*}=\mu_{\displaystyle{\pi}_k*\theta}$
% start with the optimization problem achieves strong duality. Thus $\tilde{\pi}^*_k=\tilde{\pi}^\theta_k$
% 
% \jy{concave and strong duality}
% We assume that the optimization problem in Eq.~\ref{apdx_eqn:mce_irl_multi_inner} is feasible
% . Because the objective function in Eq.~\ref{apdx_eqn:mce_irl_multi_inner_obj} is bounded, 
% the primal problem has an optimal solution when it is feasible. 
We first show that when the optimization problem is feasible, the optimal solution to the optimization problem in Eq.~\ref{apdx_eqn:mce_irl_multi_inner} are the visitation variables induced by the expert policies (i.e. $\mu_{\tilde{\pi}_k}^{\gamma_k}(s,a)=\mu_{\tilde{\pi}^*_k}^{\gamma_k}(s,a)$). Thus, Theorem 2 in~\cite{syed2008apprenticeship}, 
$    \displaystyle\tilde\pi_k(a|s) = \displaystyle\tilde{\pi}^*_k(a|s).$

By strong duality, if the primal problem has an optimal solution, the Lagrangian dual problem also has an optimal solution, which we denote by
 $\hat{\theta},\{\mu_{\hat{\pi}_k}^{\gamma_k}(s,a)\}_{k=1}^K$. Previous work shows that, $\mu_{\hat{\pi}_k}^{\gamma_k}(s,a)$ are the visitation variables induced by the entropy-regularized optimal policy, $\hat{\pi}_k$, under reward parameters $\hat\theta$ and the discount factor $\gamma_k$~\citep{zhou2017infinite}. 

By proof by contradiction, assume that there exists a $k$ such that $\hat{\pi}_k\neq \tilde{\pi}^*_k$.
By strong duality, an optimal solution to the dual problem satisfies the constraints of the primal problem. Thus, $f_{\hat{\pi}_k}^{\gamma_k} = f_{\tilde{\pi}_k^*}^{\gamma_k}$. 
Because $\hat{\pi}_k$ is a unique optimal policy for reward parameters $\hat{\theta}$~\citep{haarnoja2017reinforcement}, we have that
\begin{align*}
\tilde{V}^{\hat{\theta},\gamma_k}_{\hat{\pi}_k}&>\tilde{V}^{\hat{\theta},\gamma_k}_{\tilde{\pi}^*_k}\\
{\hat{\theta}}^\intercal f_{\hat{\pi}_k}^{\gamma_k}+\mathcal{H}^{\gamma_k}_{\hat{\pi}_k}&>{\hat{\theta}}^\intercal f_{\tilde{\pi}_k}^{\gamma_k}+\mathcal{H}^{\gamma_k}_{\tilde{\pi}_k}\\
\mathcal{H}^{\gamma_k}_{\hat{\pi}_k}&>\mathcal{H}^{\gamma_k}_{\tilde{\pi}_k}.
\end{align*}
This implies that for any reward parameters $\theta$, $\tilde{\pi}^*_k$ cannot be optimal because
\begin{equation}
\label{apdx_eqn:always_more_optimal}
\tilde{V}^{{\theta},\gamma_k}_{\hat{\pi}_k}=\theta^\intercal f_{\hat{\pi}_k}^{\gamma_k}+\mathcal{H}^{\gamma_k}_{\hat{\pi}_k}=\theta^\intercal f_{\tilde{\pi}_k}^{\gamma_k}+\mathcal{H}^{\gamma_k}_{\hat{\pi}_k}>\tilde{V}^{{\theta},\gamma_k}_{\tilde{\pi}^*_k}.
\end{equation}
 
% under true reward function
However, we can construct a reward function as $$r'=r^*+(\gamma^*_k-\gamma_k) \tilde{V}^{r^*,\gamma_k^*}_{\tilde{\pi}_k}=\left(\theta +\frac{(\gamma^*_k-\gamma_k) \tilde{V}^{r^*,\gamma_k^*}_{\tilde{\pi}_k}}{\sum_{s}\phi^s(s)}\mathbf{1}\right)^\intercal\phi^s.$$ By Theorem 5 in~\cite{cao2021identifiability}, we have that  
\begin{align*}
    T^a r'  &= T^a (r^*+(\gamma^*_k-\gamma_k) \tilde{V}^{r^*,\gamma_k^*}_{\tilde{\pi}_k})\\
    & = T^a (r^*+\gamma^*_k \tilde{V}^{r^*,\gamma_k^*}_{\tilde{\pi}_k})- \gamma_k T^a \tilde{V}^{r^*,\gamma_k^*}_{\tilde{\pi}_k}\\
    & = \lambda \log \tilde{\pi}^*_k(a|\cdot) - \gamma_k T^a \tilde{V}^{r^*,\gamma_k^*}_{\tilde{\pi}_k}+ \tilde{V}^{r^*,\gamma_k^*}_{\tilde{\pi}_k}
\end{align*}
Thus, policy $\tilde{\pi}_k$ is optimal for reward function $r'$ and the discount factor $\gamma_k$, which conflicts Eq.~\ref{apdx_eqn:always_more_optimal}.

Thus,  if the optimization problem is feasible,
$\mu_{\hat{\pi}_k}^{\gamma_k}(s,a)=\mu_{\displaystyle\tilde{\pi}^*_k}^{\gamma_k}(s,a)$ and $\hat{\pi}_k = \tilde{\pi}^*_k$. 
By Proposition~\ref{prop:id}, there exists Lagrangian multipliers $\theta$ such that $\hat{\pi}_k=\Tilde{\pi}^*_k$ if $rank(\Phi\vert b)=rank(\Phi)$.

On the other hand, if $rank(\Phi\vert b)=rank(\Phi)$, then there exist reward parameters $\theta$ such that $\tilde{\pi}^*_k$ is optimal for $\theta$ and $\gamma_k$.
We can see that the solution
$\mu_{\tilde{\pi}^*_k}^{\gamma_k}(s,a)$  satisfy the KKT conditions. By strong duality, the primal problem is feasible and has an optimal solution, $\mu_{\tilde{\pi}^*_k}^{\gamma_k}(s,a)$.
\end{proof}

%  $$g(\Gamma,r) = \min_{{k\in[K],(s,a)\in\Omega_k}} Q^{r,\gamma_k}_{\pi^*_k}(s,\pi^*_k(s)) - Q^{r,\gamma_k}_{\pi^*_k}(s,a)-\lambda\|r\|_1.$$
% We rewrite the optimization problem in Eq.~\ref{eqn:multi_gamma_cvx_irl} as,
% \begin{equation*}
%     \max_{\Gamma\in[0,1]^k} g^*,
%     \label{eqn:upper_level_obj}
% \end{equation*}
% where $g^*=\max_r g(\Gamma,r)$ under constraints ~\ref{eqn:lp_irl_obj_const1}
\section{Algorithm}
\label{apdx_sec:alg}
In Algorithm~\ref{alg:lp} and~\ref{alg:mceirl}, we give the pseudocode of MPLP-IRL and MPMCE-IRL algorithms. 
\begin{algorithm}
\caption{Mutli-planning horizon LP-IRL (MPLP-IRL)}
\begin{algorithmic}[1]
  \STATE Given a set of $K$ observed policies $\Pi^*=\{\pi^*_k\}_{k=1}^K$, the transition dynamics $T$.
  \STATE Place a Gaussian process prior on $g^*(\Gamma)$ (Eq.~\ref{eqn:upper_level_obj}). 
  \STATE Observe $g^*$ at a set of $m$ points, $\{\Gamma^{(i)}\}_{i=1}^m$, with $\Gamma^{(i)}\sim \texttt{Unif}(0,1)^K$
  \WHILE{$m \leq \text{MaxIter}$}
  \STATE Update the posterior distribution of $g^*$ given all observed points.
  \STATE Query a new point $\Gamma^{(m)}$ based on the acquisition function.
  \STATE Observe $g^*_m = \max_r g(\Gamma^{(m)},r)$ with constraints in Eq.~\ref{eqn:lp_irl_obj_const1},~\ref{eqn:lp_irl_obj_const2}.
  \STATE increment $m$
  \ENDWHILE
  \STATE Return the reward function $r$ and the set of discount factors $\Gamma$ with the largest observed $g^*$.
\end{algorithmic}
\label{alg:lp}
\end{algorithm}
\begin{algorithm}
\caption{Mutli-planning horizon MCE-IRL (MCELP-IRL)}
\begin{algorithmic}[1]
  \STATE Given a set of $K$ observed policies $\tilde{\Pi}^*=\{\tilde{\pi}^*_k\}_{k=1}^K$ and the transition dynamics $T$.
  \STATE Place a Gaussian process prior on $g^*(\Gamma)$ (Eq.~\ref{eqn:mce_inner}).
  \STATE Observe $g^*$ at a set of $m$ points, $\{\Gamma^{(i)}\}_{i=1}^m$, with $\Gamma^{(i)}\sim \texttt{Unif}(0,1)^K$
  \WHILE{$m \leq \text{MaxIter}$}
  \STATE Update the posterior distribution of $g^*$ given all observed points.
  \STATE Query a new point $\Gamma^{(m)}$ based on the acquisition function.
  \STATE Solve the Lagrangian dual of the constrained optimization problem in Eq.~\ref{eqn:mce_inner} using the ML-IRL algorithm in~\cite{zeng2022maximum}.
  \STATE Denote the optimal solution to the Lagrangian dual problem as $\tilde\theta$. Calculate the duality gap as $l=\sum_{k=1}^K \tilde\theta^\intercal(f^{\gamma_k}_{\displaystyle\tilde{\pi}^{\displaystyle{\tilde\theta}}_k}-f^{\gamma_k}_{\displaystyle\tilde{\pi}^*_k})$
  \IF{$l\leq\epsilon$}
  \STATE $g_m^*(\Gamma)=\sum_{k=1}^K \mathcal{H}^{\gamma_k}_{\displaystyle\tilde{\pi}^{\displaystyle{\tilde\theta}}_k}$
  \ELSE
   \STATE $g_m^*(\Gamma)=-|l|$
  \ENDIF
  \STATE increment $m$
  \ENDWHILE
  \STATE Return the reward parameters $\theta$ and the set of discount factors $\Gamma$ with the largest observed $g^*$.
\end{algorithmic}
\label{alg:mceirl}
\end{algorithm}
\section{Domains}
\label{apdx_sec:domains}
\subsection{Toy Domain with a Discrete MDP}\label{apdx_sec:toy_domain}
The toy domain is designed such that the optimal policy when solving standard MDPs is a $3$-step piecewise function with respect to the discount factor $\gamma\in[0,1]$. 

\begin{figure}[htbp]
    \centering
\begin{subfigure}[t]{0.4\linewidth}
  \resizebox{\textwidth}{!}{
  \begin{tikzpicture}
    [->,align=center,node distance=3cm] 
    \node[state, label=$+0$] (s0) {$s_0$};
    \node[state, right of=s0, label={\color{orange}$+6$}] (s1) {$s_1$};
    \node[state, above of=s1, label={\color{purple}$+7$}] (s2) {$s_2$};
    \node[state, accepting, right of=s1,, label={$+10$}] (s3) {$s_3$};
    \draw (s0) edge[bend right,below] node{\color{blue}$a_0$ w.p. $0.95$} (s3)
    (s0) edge[above] node{\color{orange}$a_1$ w.p. $0.9$} (s1)
    (s0) edge[above] node{\color{purple}$a_2$ w.p. $0.6$} (s2)
    (s1) edge[above] (s3)
    (s2) edge[above] (s3)
    ;
    \end{tikzpicture}
  }
  \caption{Discrete MDP with the true reward function.}
  \label{fig:mdp_true_r}
\end{subfigure}
~
\begin{subfigure}[t]{0.4\linewidth}
  \resizebox{\textwidth}{!}{
  \begin{tikzpicture}
    [->,align=center,node distance=3cm] 
    \node[state, label=$+0$] (s0) {$s_0$};
    \node[state, right of=s0, label={\color{blue}$+8.83$}] (s1) {$s_1$};
    \node[state, above of=s1, label={\color{blue}$+10$}] (s2) {$s_2$};
    \node[state, accepting, right of=s1,, label={$+10$}] (s3) {$s_3$};
    \draw (s0) edge[bend right,below] node{$a_0$ w.p. $0.95$} (s3)
    (s0) edge[above] node{$a_1$ w.p. $0.9$} (s1)
    (s0) edge[above] node{$a_2$ w.p. $0.6$} (s2)
    (s1) edge[above] (s3)
    (s2) edge[above] (s3)
    ;
    \end{tikzpicture}
  }
  \caption{Discrete MDP with the learned function of MPLP-IRL.}
  \label{fig:lp_learned_r}
  \end{subfigure}
  ~
\begin{subfigure}[t]{0.4\linewidth}
  \resizebox{\textwidth}{!}{
  \begin{tikzpicture}
    [->,align=center,node distance=3cm] 
    \node[state, label={\color{blue}$+3.4$}] (s0) {$s_0$};
    \node[state, right of=s0, label={\color{blue}$+4.2$}] (s1) {$s_1$};
    \node[state, above of=s1, label={\color{blue}$+3.1$}] (s2) {$s_2$};
    \node[state, accepting, right of=s1,, label={$+9.7$}] (s3) {$s_3$};
    \draw (s0) edge[bend right,below] node{$a_0$ w.p. $0.95$} (s3)
    (s0) edge[above] node{$a_1$ w.p. $0.9$} (s1)
    (s0) edge[above] node{$a_2$ w.p. $0.6$} (s2)
    (s1) edge[above] (s3)
    (s2) edge[above] (s3)
    ;
    \end{tikzpicture}
  }
  \caption{Discrete MDP with the learned function of MPMCE-IRL.}
  \label{fig:mce_learned_r}
  \end{subfigure}
  ~
  \begin{subfigure}[t]{0.4\linewidth}
  \resizebox{\textwidth}{!}{
  \begin{tikzpicture}
    [->,align=center,node distance=3cm] 
    \node[state, label={\color{blue}$+0$}] (s0) {$s_0$};
    \node[state, right of=s0, label={\color{blue}$+0$}] (s1) {$s_1$};
    \node[state, above of=s1, label={\color{blue}$+0$}] (s2) {$s_2$};
    \node[state, accepting, right of=s1,, label={$+10$}] (s3) {$s_3$};
    \draw (s0) edge[bend right,below] node{$a_0$ w.p. $0.95$} (s3)
    (s0) edge[above] node{$a_1$ w.p. $0.9$} (s1)
    (s0) edge[above] node{$a_2$ w.p. $0.6$} (s2)
    (s1) edge[above] (s3)
    (s2) edge[above] (s3)
    ;
    \end{tikzpicture}
  }
  \caption{Discrete MDP with the learned function of naive LP-IRL.}
  \label{fig:naive_lp_learned_r}
  \end{subfigure}
  \caption{Toy Domain}
  \label{fig:toy}
  \end{figure}
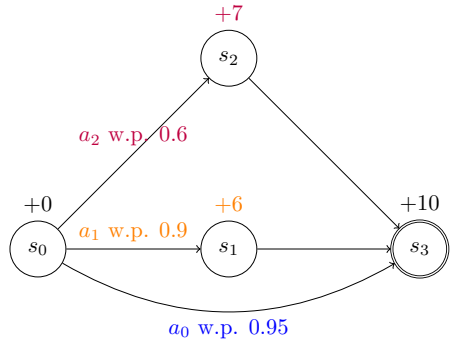
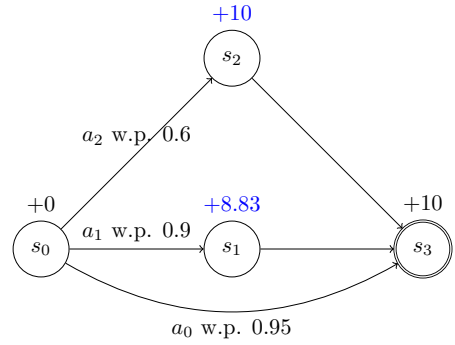
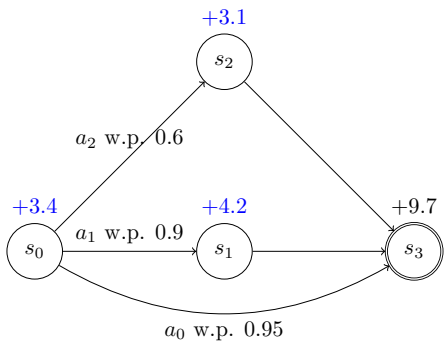
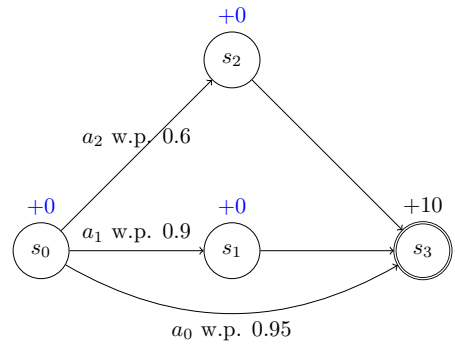

The MDP is described in Fig.~\ref{fig:mdp_true_r}.
Specifically, $s_0$ is the initial state, $s_3$ is the absorbing state. The agent gets a large positive reward when getting to the absorbing state ($r^*(s_3)=10$). The agent gets some small rewards when getting to $s_1$, $s_2$ ($r^*(s_1) = 6,\ r^*(s_2) = 7$). 
With action $a_0$, the agent moves to the absorbing state $s_3$ with probability (w.p.) $0.95$. With action $a_1$, the agent moves to $s_1$ w.p. $0.9$. With action $a_2$, the agent moves to $s_2$ w.p. $0.6$. The agent stays otherwise. Note that although $s_2$ has a larger reward, the agent faces more stochasticity. Thus, there is a trade-off when the discount factor $\gamma$ varies.

\paragraph{Expert policies for MPLP-IRL.}
Let $\pi^*_0, \pi^*_1, \pi^*_2$ denote the optimal policies where the optimal actions are $a_0, a_1, a_2$ ($\pi^*_i = \mathbbm{1}_{a_i}$), respectively. We note that for this toy example, in state $s_1$, $s_2$, $s_3$, $\pi^*_0, \pi^*_1, \pi^*_2$ are equally optimal. Thus, we focus on $s_0$ in the later analysis. 

\begin{align*}
\begin{cases}
    V^{\pi_0}_\gamma (s_0)&= \frac{0.95\cdot 10}{1-0.05 \gamma}\\
    V^{\pi_1}_\gamma(s_0) &= \frac{0.9(6+10\gamma)}{1-0.1\gamma}\\
    V^{\pi_2}_\gamma (s_0)&= \frac{0.6(7+10\gamma)}{1-0.4\gamma}
\end{cases}
\end{align*}
Solve pairwise difference between value functions and let $\gamma_0\approx 0.432,\ \gamma_1\approx 0.876$. When $\gamma<\gamma_0$, $\pi^*=\pi_0$. When $\gamma_0<\gamma<\gamma_1$, $\pi^*=\pi_1$. When $\gamma>\gamma_1$, $\pi^*=\pi_2$. The value functions of these $3$ expert policies, $\Pi^*$, under different discount values are shown in~\ref{fig:toy_opt_policy}.

\paragraph{Satisfaction of Assumption~\ref{assump_1}.} 
Figure~\ref{apdx_fig:toy_opt_polic} plots the value function of all three expert policies of $s_0$ under the true reward function $r^*$ with respect to discount factors $\gamma\in[0,1]$. Assumption~\ref{assump_1} is violated ($Q_{\pi_i^*}^{r^*,\gamma^*_i}(s,\pi_i^*(s)) = Q_{\pi_i^*}^{r^*,\gamma^*_i}(s,\pi_j^*(s))$) if and only if any of the two lines intersect. We see that across the entire domain ($[0,1]^3$), there are only three sets of discount factors that violate Assumption~\ref{assump_1}. Thus in practice, the likelihood of violating Assumption~\ref{assump_1} is low.

\paragraph{Expert policies for MPMCE-IRL.}
For MPMCE-IRL, we observe expert demonstrations from $3$ expert policies under discount factors $\Gamma^* =\{0.3,0.5,0.95\}$.
% \subsection{Discrete MDP for Identifiability Analyses}
% \label{apdx_subsec:id_domain}

\subsection{Big-Small Domain}
In the big-small domain, there are two absorbing states: one at the left bottom cell with a small reward ($+2$) and one at the right bottom cell with a large reward ($+20$). Each step costs $-2$ if the agent does not reach the absorbing state. The true reward function is plotted in~\ref{fig:big_small_true_r}. The agent can start from anywhere in the grid world except for the absorbing states. The agent can choose to move \{\texttt{right, down, left, up}\} at each state. If the agent comes across the wall by taking an action, the agent stays where it is and moves to the corresponding state otherwise. The transition dynamics of the grid domain are deterministic.

\paragraph{Expert policies for MPLP-IRL.}
We observe expert demonstrations from $3$ distinct expert policies whose optimal actions are visualized in Fig.~\ref{fig:big_small_policy}. We see that 
when the discount factor is small, the agent will go for the closest reward regardless of its magnitude (Fig.~\ref{fig:big_small_policy_0}). When the discount factor is large, the agent prefers the large reward regardless of how far the reward is (Fig.~\ref{fig:big_small_policy_2}).

\paragraph{Expert policies for MPMCE-IRL.}
For MPMCE-IRL, we observe expert demonstrations from $3$ expert policies under discount factors $\Gamma^* =\{0.1,0.45,0.9\}$

\begin{figure}[htbp]
\centering
\makebox[\textwidth]{
\begin{subfigure}[t]{0.32\paperwidth}
\includegraphics[width=\textwidth]{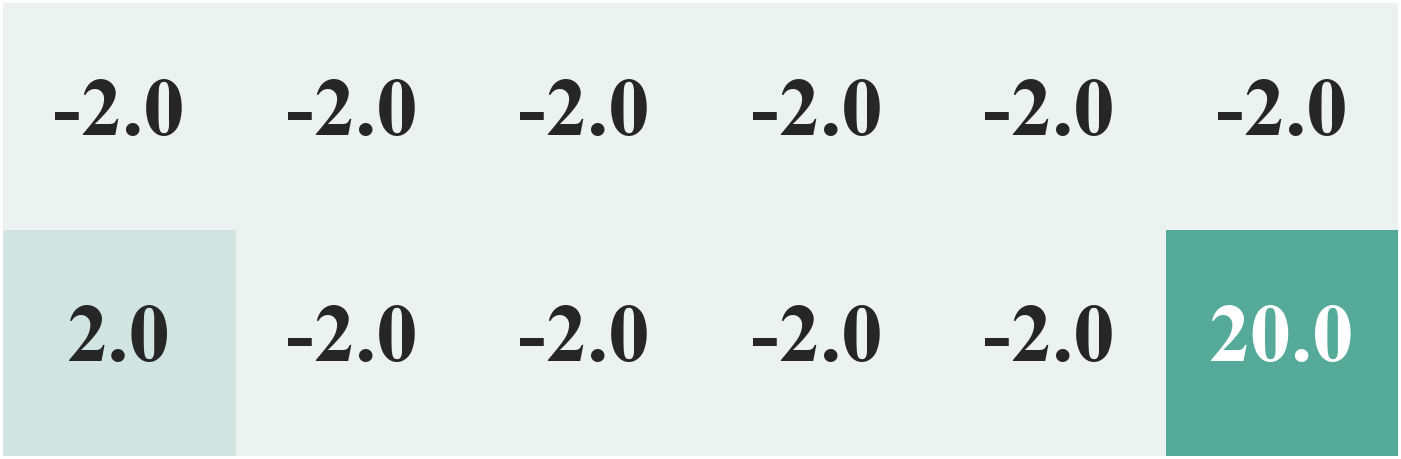}
\caption{The true reward function $r^*$.}
\label{fig:big_small_true_r}
\end{subfigure}
~
\begin{subfigure}[t]{0.32\paperwidth}
    \includegraphics[width=\textwidth]{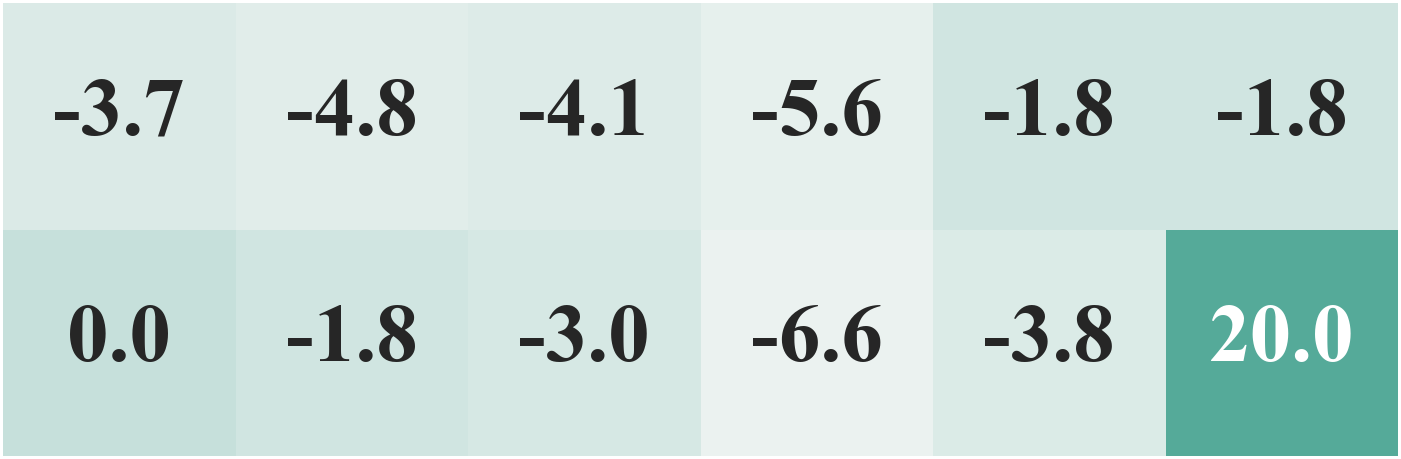}
    \caption{The learned reward function of MPLP-IRL.}
    \label{fig:big_small_lp_r}
\end{subfigure}
~
\begin{subfigure}[t]{0.32\paperwidth}
    \includegraphics[width=\textwidth]{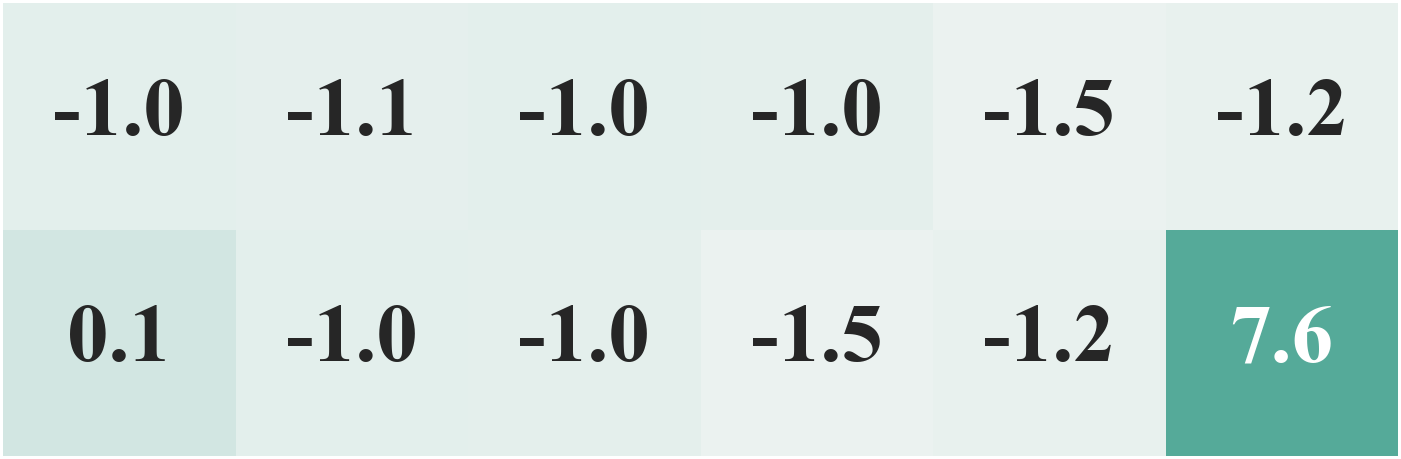}
    \caption{The learned reward function of MPMCE-IRL.}
    \label{fig:big_small_mce_r}
\end{subfigure}
}
\caption{The reward function of the big-small domain of each state.}
\label{fig:big_small_rs}
\end{figure}
\begin{figure}[htbp]
\centering
\makebox[\textwidth]{
    \begin{subfigure}[t]{0.32\paperwidth}
    \includegraphics[width=\textwidth]{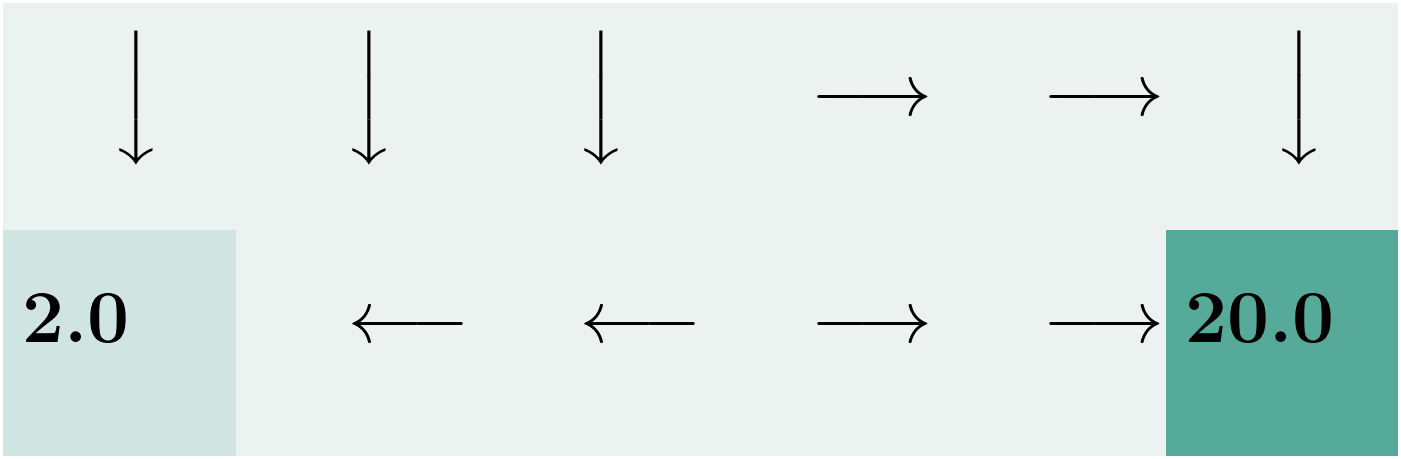}
    \caption{The optimal actions of expert policy $\pi^*_0,\ (\gamma^*_0\in[0,0.2])$}
    \label{fig:big_small_policy_0}
    \end{subfigure}
    ~
    \begin{subfigure}[t]{0.32\paperwidth}
    \includegraphics[width=\textwidth]{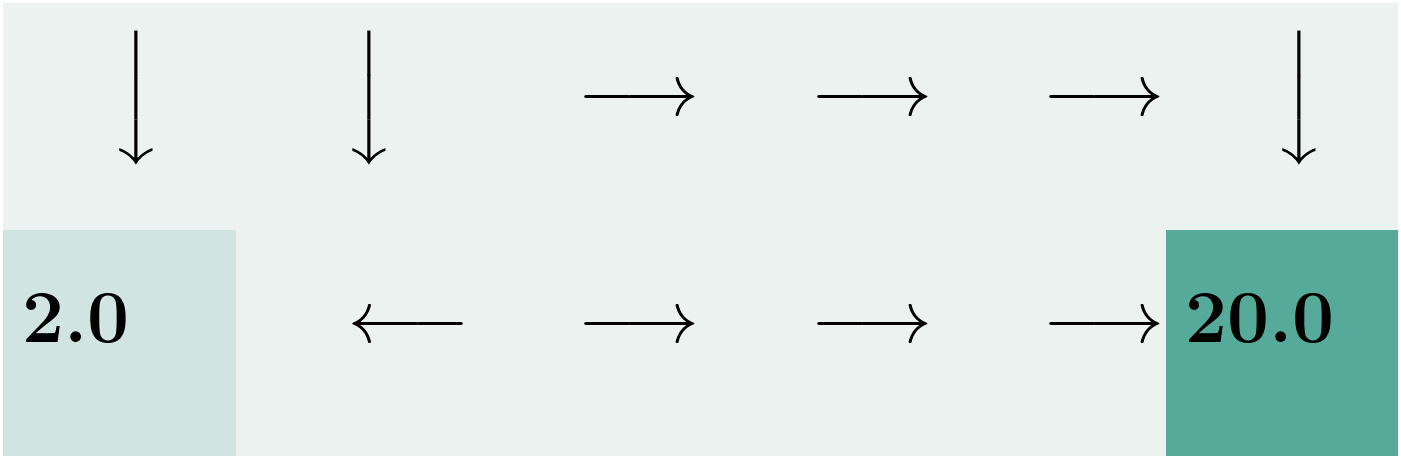}
    \caption{The optimal actions of expert policy $\pi^*_1,\ (\gamma^*_1\in(0.2,0.68])$}
    \label{fig:big_small_policy_1}
    \end{subfigure}
    ~
    \begin{subfigure}[t]{0.32\paperwidth}
    \includegraphics[width=\textwidth]{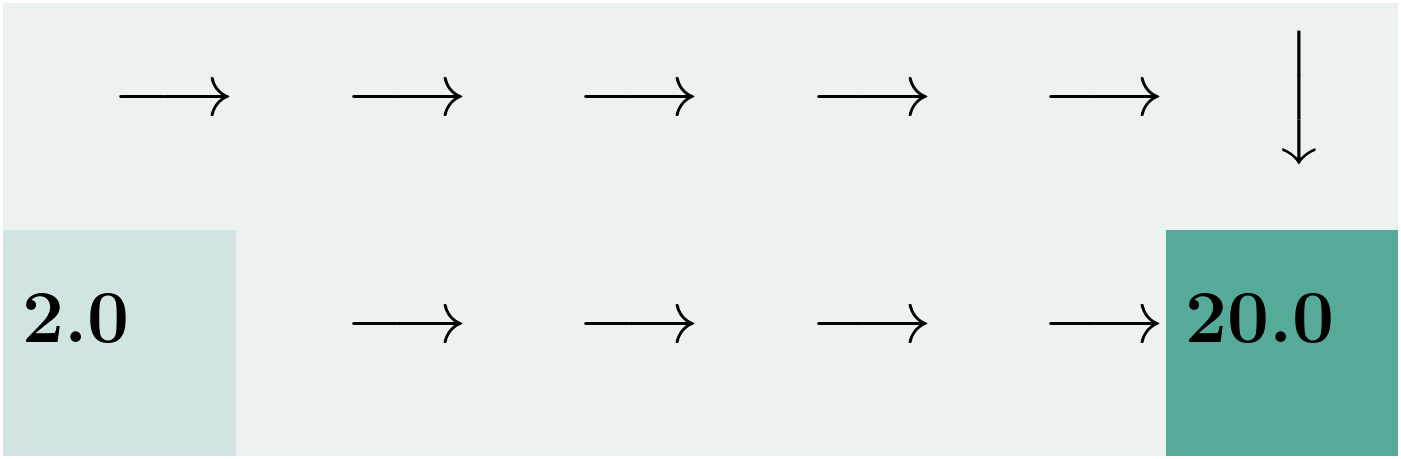}
    \caption{The optimal actions of expert policy $\pi^*_2,\ (\gamma^*_2\in[0.87,1])$}
    \label{fig:big_small_policy_2}
    \end{subfigure}
}
\caption{Visualization of the optimal policies for standard MDPs of the big-small domain.}
\label{fig:big_small_policy}
\end{figure}

\subsection{Cliff Domain}
In the cliff domain, the trajectory ends whenever the agent falls off the cliff (the top row in Fig~\ref{fig:cliff_true_r}. There is one rewarding state at the upper right ($+20$). The agent gets a large penalty whenever it falls off the cliff ($-10$). There is a small cost for each step ($-2$ when the agent is close to the cliff and $-1$ when the agent is far away from the cliff). For each action,
the agent moves in the intended direction with probability $0.9$, and moves in a random other direction with probability $0.1$.

\paragraph{Expert policies for MPLP-IRL.}
We observe $3$ distinct expert policies whose optimal actions are visualized in Fig.~\ref{fig:cliff_policy}. We see that 
when the discount factor is small, the agent is not willing to risk and tries to avoid walking along the cliff (Fig.~\ref{fig:cliff_policy_0}).
When the discount factor is large, the agent risks falling off the cliff for good returns (Fig.~\ref{fig:cliff_policy_2}).

\paragraph{Expert policies for MPMCE-IRL.} For MPMCE-IRL, we observe expert demonstrations from $3$ expert policies under discount factors $\Gamma^* =\{0,0.2,0.52\}$.

\begin{figure}[htbp]
\centering
\begin{subfigure}[t]{0.3\linewidth}
\includegraphics[width=\textwidth]{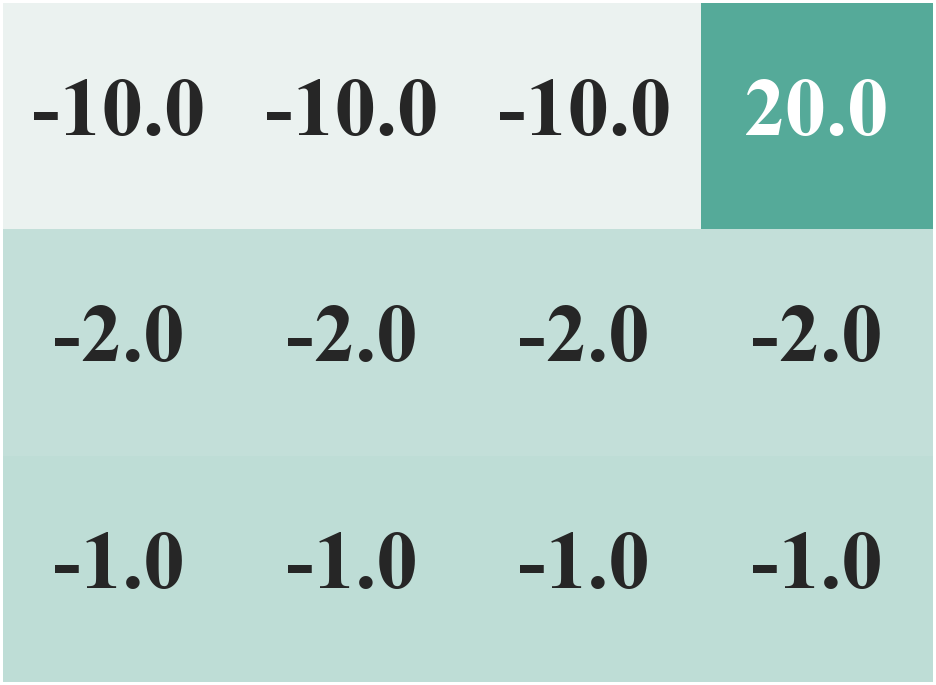}
\caption{The true reward function $r^*$.}
\label{fig:cliff_true_r}
\end{subfigure}
~
\begin{subfigure}[t]{0.3\linewidth}
    \includegraphics[width=\textwidth]{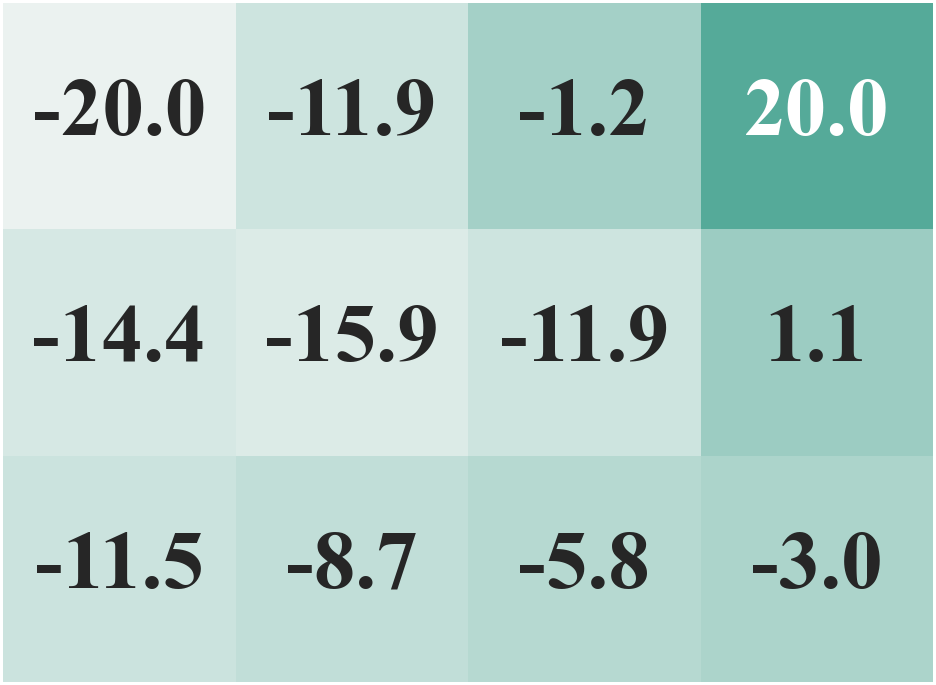}
    \caption{The learned reward function of MPLP-IRL.}
    \label{fig:cliff_lp_r}
\end{subfigure}
~
\begin{subfigure}[t]{0.3\linewidth}
    \includegraphics[width=\textwidth]{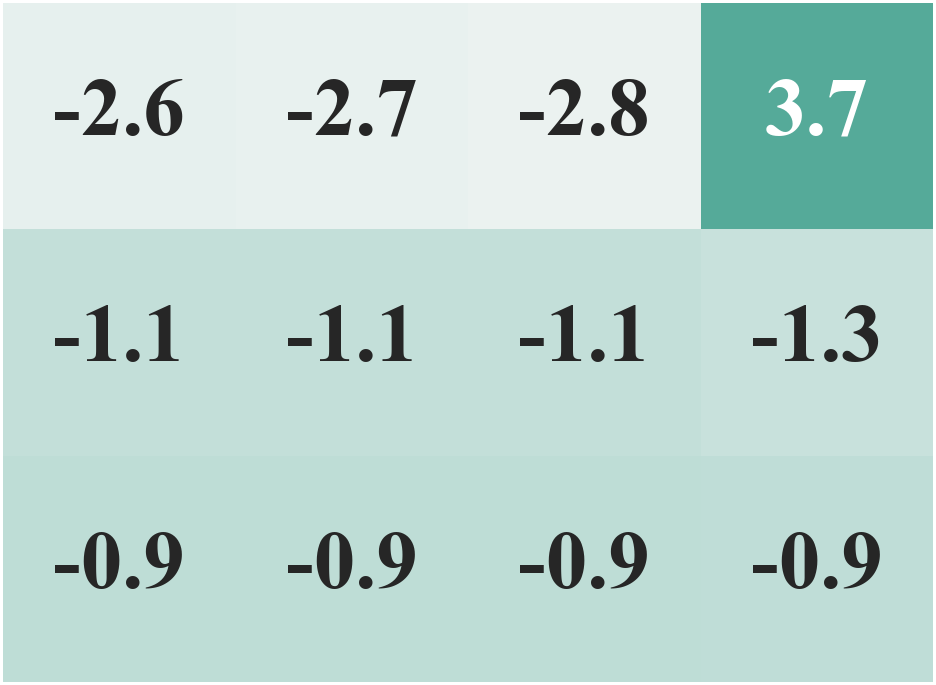}
    \caption{The learned reward function of MPMCE-IRL.}
    \label{fig:cliff_mce_r}
\end{subfigure}
\caption{The reward function of the cliff domain of each state.}
\label{fig:cliff_rs}
\end{figure}
\begin{figure}[htbp]
\centering
    \begin{subfigure}[t]{0.3\linewidth}
    \includegraphics[width=\textwidth]{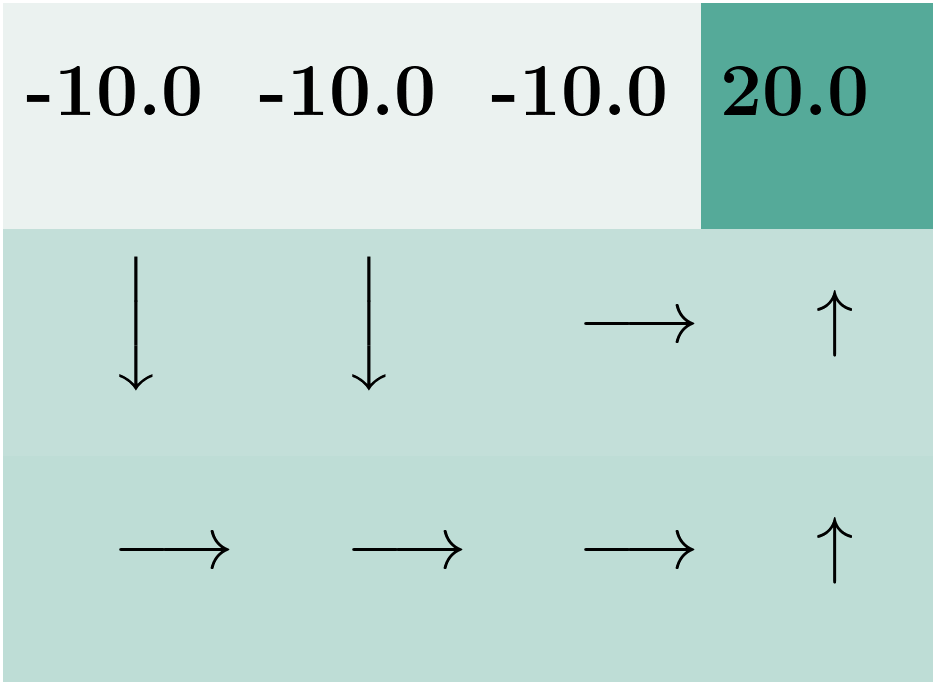}
    \caption{The optimal actions of expert policy $\pi^*_0,\ (\gamma^*_0\in[0.06,0.28])$}
    \label{fig:cliff_policy_0}
    \end{subfigure}
    ~
    \begin{subfigure}[t]{0.3\linewidth}
    \includegraphics[width=\textwidth]{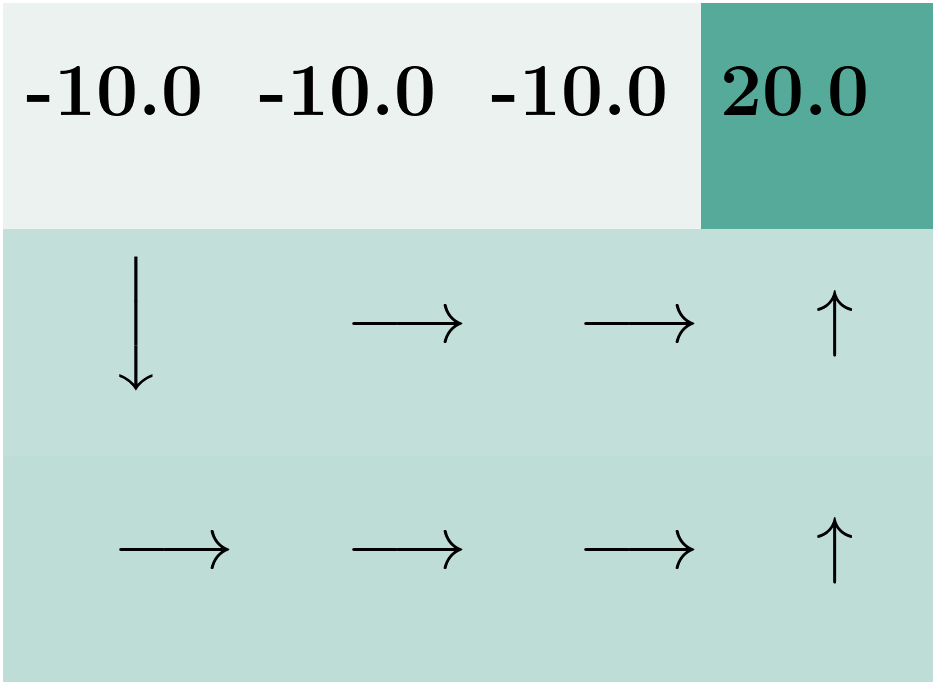}
    \caption{The optimal actions of expert policy $\pi^*_1,\ (\gamma^*_1\in(0.28,0.52])$}
    \label{fig:cliff_policy_1}
    \end{subfigure}
    ~
    \begin{subfigure}[t]{0.3\linewidth}
    \includegraphics[width=\textwidth]{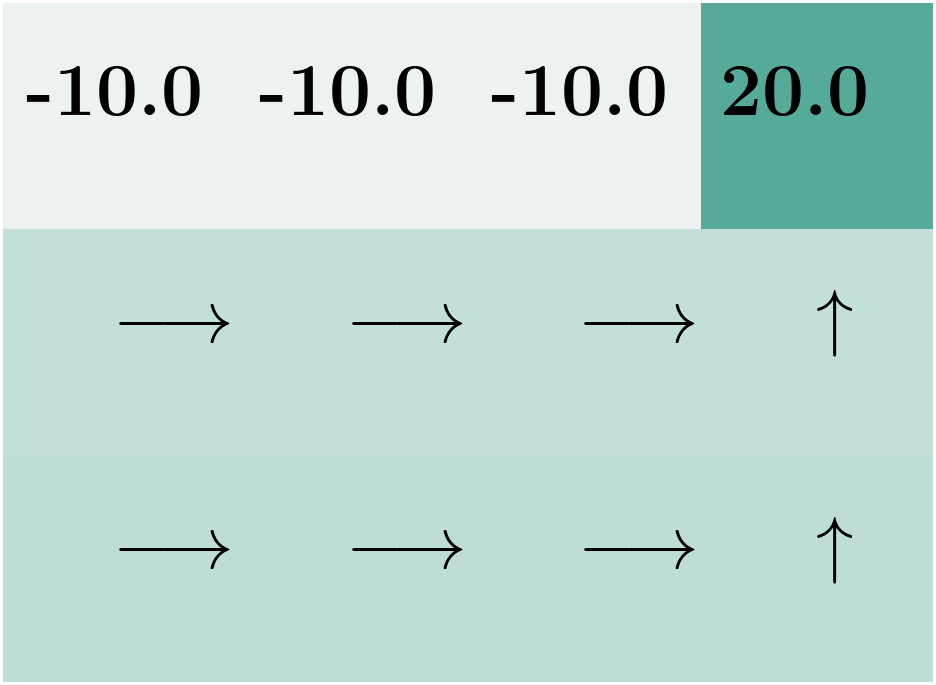}
    \caption{The optimal actions of expert policy $\pi^*_2,\ (\gamma^*_2\in(0.52,0.95])$}
    \label{fig:cliff_policy_2}
    \end{subfigure}
\caption{Visualization of the optimal policies for standard MDPs of the cliff domain.}
\label{fig:cliff_policy}
\end{figure}
\section{Generalization Error}
\label{sec:gen_error}
\subsection{Generation of random environments}\label{sec:gen_random_env}
 To test the generalizability of the learned reward function of our algorithms, for each domain environment, we generate $N=100$ random transition dynamics $\{T^{(n)}\}_{n=1}^{N}$ and discount factors $\{\gamma^{(n)}\}_{n=1}^{N}$. 
 For the toy domain, we randomize the probability of the agent' moving to the next state by taking actions from state $s_0$. From state $s_1,\ s_2$, the agent still moves to the absorbing state with probability $1$ regardless of its actions.
For the big-small and cliff domains, we generate random environments by adding noise to the agent's intended direction. For each action,
the agent moves in the intended direction with probability $\epsilon\sim\textup{Unif}(0,1)$, and moves in a random other direction with probability $1-\epsilon$.
\subsection{Evaluation Metrics}
For each randomized transition dynamics and discount factor generated from Appendix~\ref{sec:gen_random_env}, 
we first solve the optimal policies under the true reward function $r^*$ and the learned reward function $\hat{r}$, denoted as $\pi^{*(n)},\hat\pi^{(n)}$ respectively. We evaluate the generalizability of the learned reward function $\hat{r}$ by computing the normalized difference of the value function under the true $r^*$ and randomly generated environments,
\begin{equation}
\triangle V_{\hat{r}}=\frac{1}{N}\sum_{n=1}^{N }\left(V_{\pi^{*(n)}}^{r^*,T^{(n)},\gamma^{(n)}} - V_{\hat\pi^{(n)}}^{r^*,T^{(n)},\gamma^{(n)}}\right)/V_{\pi^{*(n)}}^{r^*,T^{(n)},\gamma^{(n)}}
\label{eqn:eval_gen}
\end{equation}
\section{Additional Results}
\subsection{Undesirable Global Maxima of Naive Extensions of LP-IRL}\label{sec:naive_lp}
In this Section, we demonstrate that naive extensions of LP-IRL in Eq.~\ref{eqn:naive_lp_irl_multi_gamma} give an undesirable global optimum. In Fig.~\ref{apdx_fig:toy_opt_polic}, we see that under the true reward function, $r^*$, each expert policy is optimal (when the corresponding line is on the top) for a continuous interval. However, in Fig.~\ref{apdx_fig:naive_lp_policy}, we see that naive LP-IRL algorithm identifies a set of discount factors, $\tilde\Gamma=\{0,1,1\}$ (the orange and the purple dashed lines intersect at $\gamma=1$). Under the learned discount factors, expert policies $\pi^*_1$ and $\pi^*_2$ are not distinguishable from each other under any reward function, which violates assumption~\ref{assump_1}.

Furthermore, the naive LP-IRL cannot recover the structure of the true reward function In Fig.~\ref{fig:naive_lp_learned_r}, we see that the naive LP-IRL assigns a reward only to the absorbing state with $r(s_1)=r(s_2)=0$. This reward function gives a generalization error (Eq.~\ref{eqn:eval_gen}) of $0.213\pm 0.218$. In Fig.~\ref{apdx_fig:naive_lp_policy},
we see that under this learned reward function,  when $\gamma<1$, the expert policy $\pi^*_0$ performs substantially better than  $\pi^*_1,\pi^*_2$ while expert policies $\pi^*_1$, $\pi^*_2$ perform similarly to each other. In contrast, under the true reward function (Fig.~\ref{apdx_fig:toy_opt_polic}), we can find a set of discount factors such that the expert policies are distinguishable from each other.

\begin{figure}[htbp]
\centering
\makebox[\textwidth]{
 \begin{subfigure}[t]{0.45\linewidth}
     % \centering
     \includegraphics[width=0.8\textwidth]{imgs/toy/toy_opt_policy.png}
     \caption{The value function $V^{\gamma,r*}_{\pi^*_k}(s_0)$ of expert policies $\pi^*_0,\ \pi^*_1,\ \pi^*_2$ under the true reward $r^*$.}
     \label{apdx_fig:toy_opt_polic}
 \end{subfigure}
 ~
  \begin{subfigure}[t]{0.45\linewidth}
     % \centering
     \includegraphics[width=0.8\textwidth]{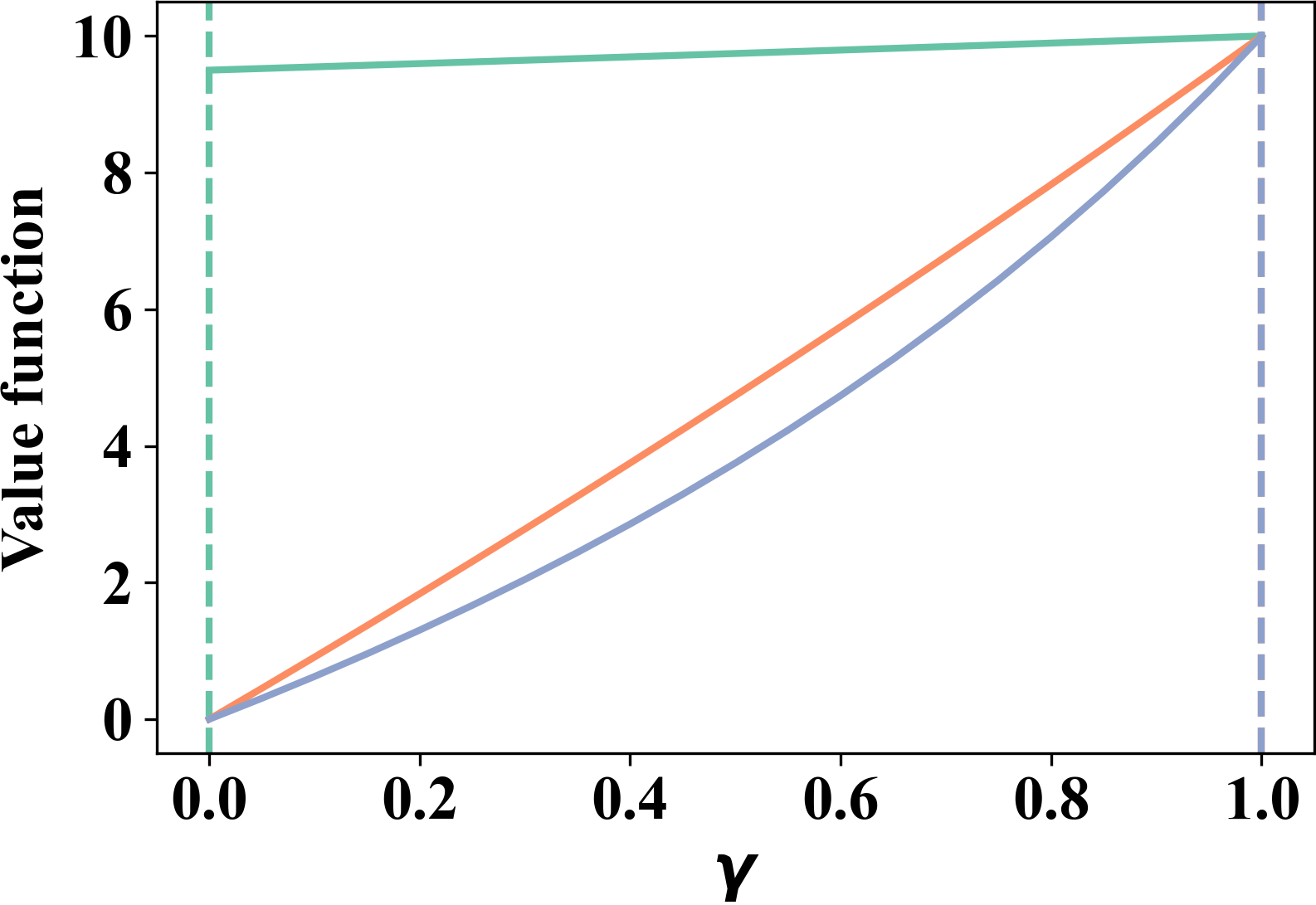}
     \caption{The value function $V^{\displaystyle{\tilde\gamma},\displaystyle{\tilde{r}}}_{\pi_k}(s_0)$ of reconstructed optimal policies $\pi_0,\ \pi_1,\ \pi_2$ under the learned reward function of the naive extension of LP-IRL (Eq.~\ref{eqn:naive_lp_irl_multi_gamma}), $\tilde{r}$.}
     \label{apdx_fig:naive_lp_policy}
 \end{subfigure}
 ~
 \begin{subfigure}{0.1\linewidth}
     % \centering
     \raisebox{3em}{
     \includegraphics[width=\textwidth]{imgs/toy/value_function_legend.png}}
 \end{subfigure}
 }
 \caption{ 
 Plots of the value function of the initial state under (a) the true reward function $r^*$, (b) the learned reward function of the naive extensions of LP-IRL, $\tilde{r}$: $x,\ y$-axes represent the discount factor $\gamma\in[0,1]$ and the value function of expert policies or reconstructed optimal policies, respectively. Each color represents a different policy. The dashed lines in (b) represent the learned discount factors, $\tilde{\Gamma}$. Each policy is optimal when the corresponding line is on the top. Policies are equally optimal when the lines interest. } 
 \label{apdx_fig:naive_lp}
 \end{figure}

% \begin{figure}[htbp]
%     \centering
%   \begin{tikzpicture}
%     [scale=0.5, ->,align=center,node distance=3cm] 
%     \node[state, label={\color{blue}$+0$}] (s0) {$s_0$};
%     \node[state, right of=s0, label={\color{blue}$+0$}] (s1) {$s_1$};
%     \node[state, above of=s1, label={\color{blue}$+0$}] (s2) {$s_2$};
%     \node[state, accepting, right of=s1,, label={$+10$}] (s3) {$s_3$};
%     \draw (s0) edge[bend right,below] node{$a_0$ w.p. $0.95$} (s3)
%     (s0) edge[above] node{$a_1$ w.p. $0.9$} (s1)
%     (s0) edge[above] node{$a_2$ w.p. $0.6$} (s2)
%     (s1) edge[above] (s3)
%     (s2) edge[above] (s3)
%     ;
%     \end{tikzpicture}
%   \caption{Discrete MDP with the learned function of naive extension of LP-IRL.}
%   \label{fig:learned_r_naive_lp}
% \end{figure}

 \subsection{Comparison of the order of true and learned discount factors} \label{sec:results_discount_factors}
 In Table~\ref{table:learned_gamma}, we provide a full comparison of the true discount factors $\Gamma^*$ and the learned discount factors $\tilde\Gamma$. Both MPLP-IRL and MPMCE-IRL recover the order of the true discount factors. 
\begin{table}[htbp]
\centering
\begin{tabular}{|l|l|l|l|l|}
\hline
\multirow{2}{*}{}          & True $\Gamma^*$  & Learned $\tilde\Gamma$ & True $\Gamma^*$ & Learned $\tilde\Gamma$ \\
                           & of MPLP-IRL      & of MPLP-IRL            & of MPMCE-IRL    & of MPMCE-IRL           \\ \hline
\multirow{3}{*}{Toy}       & $\pi^*_1$:{[}0,0.43{]}     & $\pi_1: 0$            & $\tilde\pi^*_1:0.3$  &$\tilde\pi_1:0.53$   \\
                           & $\pi^*_2$:(0.43, 0.87{]}   & $\pi_2: 0.35$         & $\tilde\pi^*_2:0.5$  &$\tilde\pi_2:0.82$   \\
                           & $\pi^*_3$:{[}0.87,1)       & $\pi_3: 1$      & $\tilde\pi^*_3:0.95$  &$\tilde\pi_3:0.98$   \\ \hline
\multirow{3}{*}{Big-Small} & $\pi^*_1$:{[}0,0.2{]}      & $\pi_1: 0.38$          & $\tilde\pi^*_1:0$  &$\tilde\pi_1:0.30$ \\
                           & $\pi^*_2$:(0.2,0.68{]}     & $\pi_3: 0.74$    & $\tilde\pi^*_2:0.45$  &$\tilde\pi_2:0.62$ \\
                           & $\pi^*_3$:{[}0.87,1{]}     & $\pi_3: 1$         & $\tilde\pi^*_3:0.9$  &$\tilde\pi_3:0.98$ \\ \hline
\multirow{3}{*}{Cliff}     & $\pi^*_1$:{[}0.06, 0.28{]} & $\pi_1: 0.04$            & $\tilde\pi^*_1:0$  &$\tilde\pi_1:0$  \\
                           & $\pi^*_2$:(0.28,0.52{]}    & $\pi_2: 0.44$          & $\tilde\pi^*_2:0.2$  &$\tilde\pi_2:0.31$ \\
                           & $\pi^*_3$:(0.52,0.95{]}    &  $\pi_3: 0.84$         & $\tilde\pi^*_3:0.52$  &$\tilde\pi_3:0.55$ \\ \hline
\end{tabular}
\caption{Table of the  true discount factors $\Gamma^*$ and the learned discount factors $\tilde\Gamma$. Both MPLP-IRL and MPMCE-IRL recover the order of the true discount factors. }
\label{table:learned_gamma}
\end{table}

\subsection{Convergence of MPLP-IRL and MPMCE-IRL}\label{sec:result_convergence}
For MPLP-IRL, we approximate the global optimum by performing a grid search on the range $[0,1]^K$. For MPMCE-IRL, it is computationally heavy to solve the optimization problem in Eq.~\ref{eqn:mce_inner} for $10^K$ times. We instead compare the best current objective value of BO to the objective value evaluated under the true reward function and discount factors. 
From  Fig.~\ref{fig:bo_trace}, we see across all domains, MPLP-IRL converges to the global maximum within $100$ iterations while MPMCE-IRL converges within $50$ iterations,

\begin{figure}[htbp]
\centering
\begin{subfigure}[t]{0.31\linewidth}
\includegraphics[width=\textwidth]{imgs/toy/lp_bo_trace.png}
\caption{MPLP-IRL: the toy domain}
\label{apdx_fig:toy_lp_bo_trace}
\end{subfigure}
~
\begin{subfigure}[t]{0.31\linewidth}
    \includegraphics[width=\textwidth]{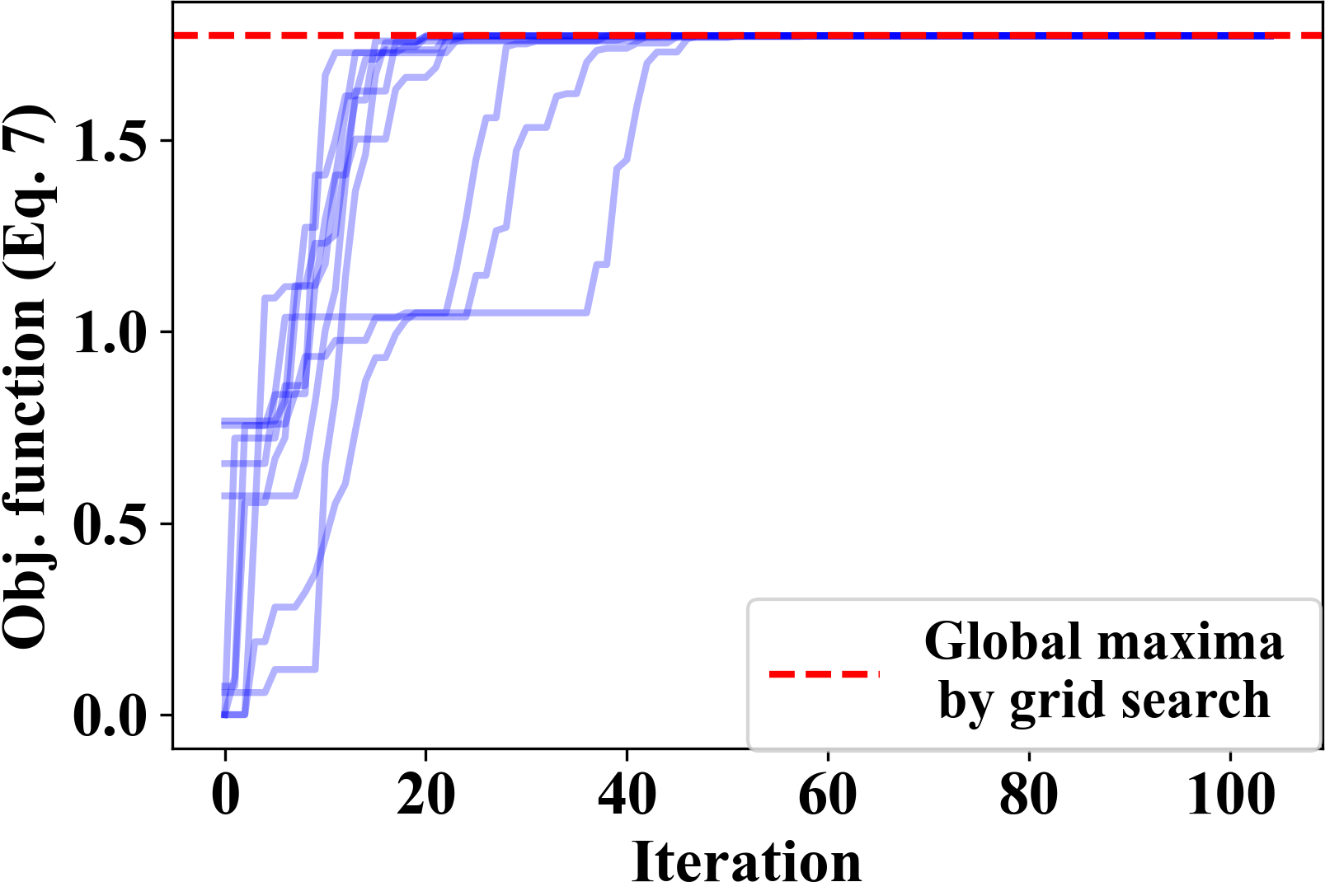}
    \caption{MPLP-IRL: the big-small domain}
    \label{apdx_fig:big_small_lp_bo_trace}
\end{subfigure}
~
\begin{subfigure}[t]{0.31\linewidth}
    \includegraphics[width=\textwidth]{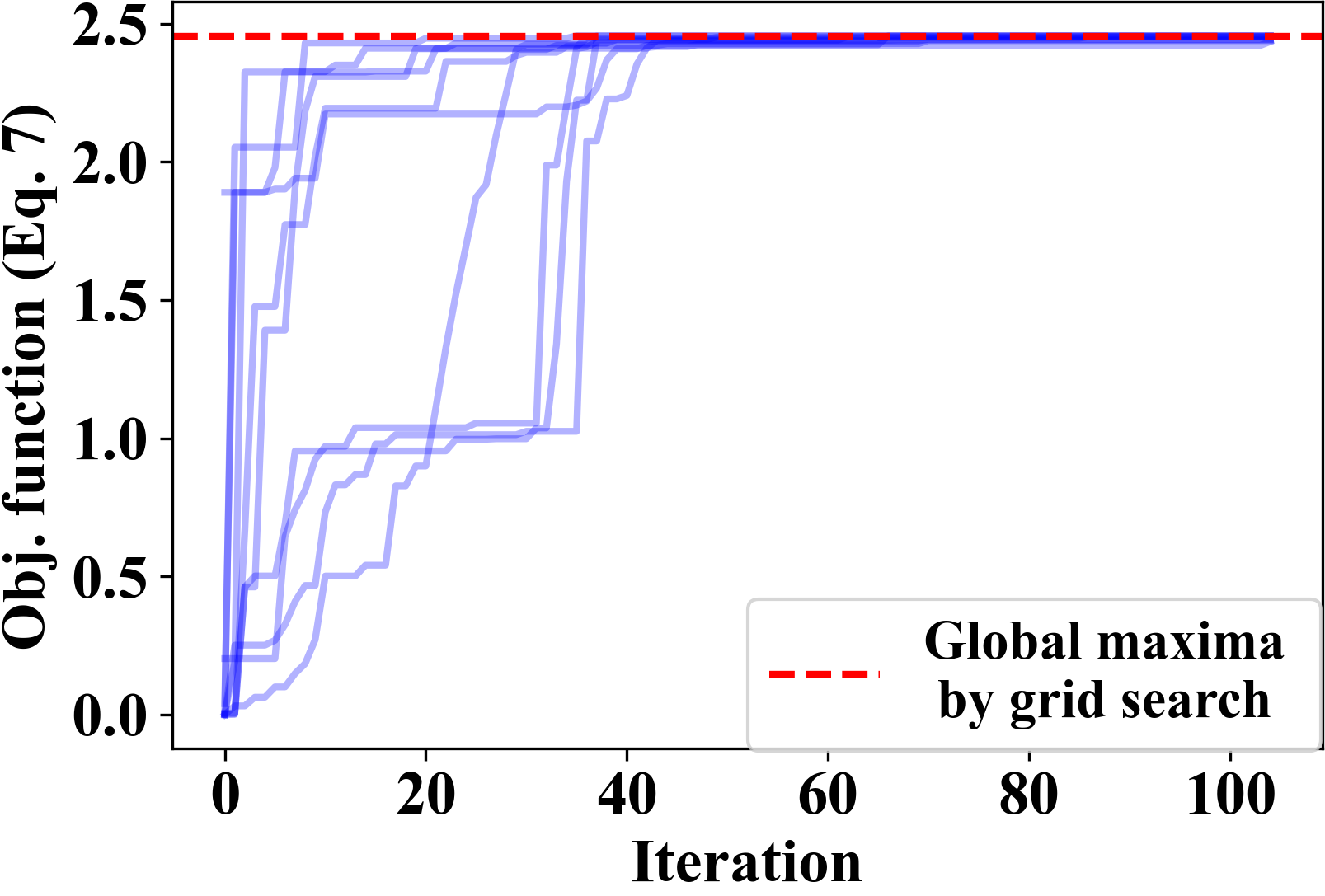}
    \caption{MPLP-IRL: the cliff domain}
    \label{apdx_fig:cliff_lp_bo_trace}
\end{subfigure}
~
\begin{subfigure}[t]{0.31\linewidth}
\includegraphics[width=\textwidth]{imgs/toy/mce_bo_trace.png}
\caption{MPMCE-IRL: the toy domain}
\label{apdx_fig:toy_mce_bo_trace}
\end{subfigure}
~
\begin{subfigure}[t]{0.31\linewidth}
    \includegraphics[width=\textwidth]{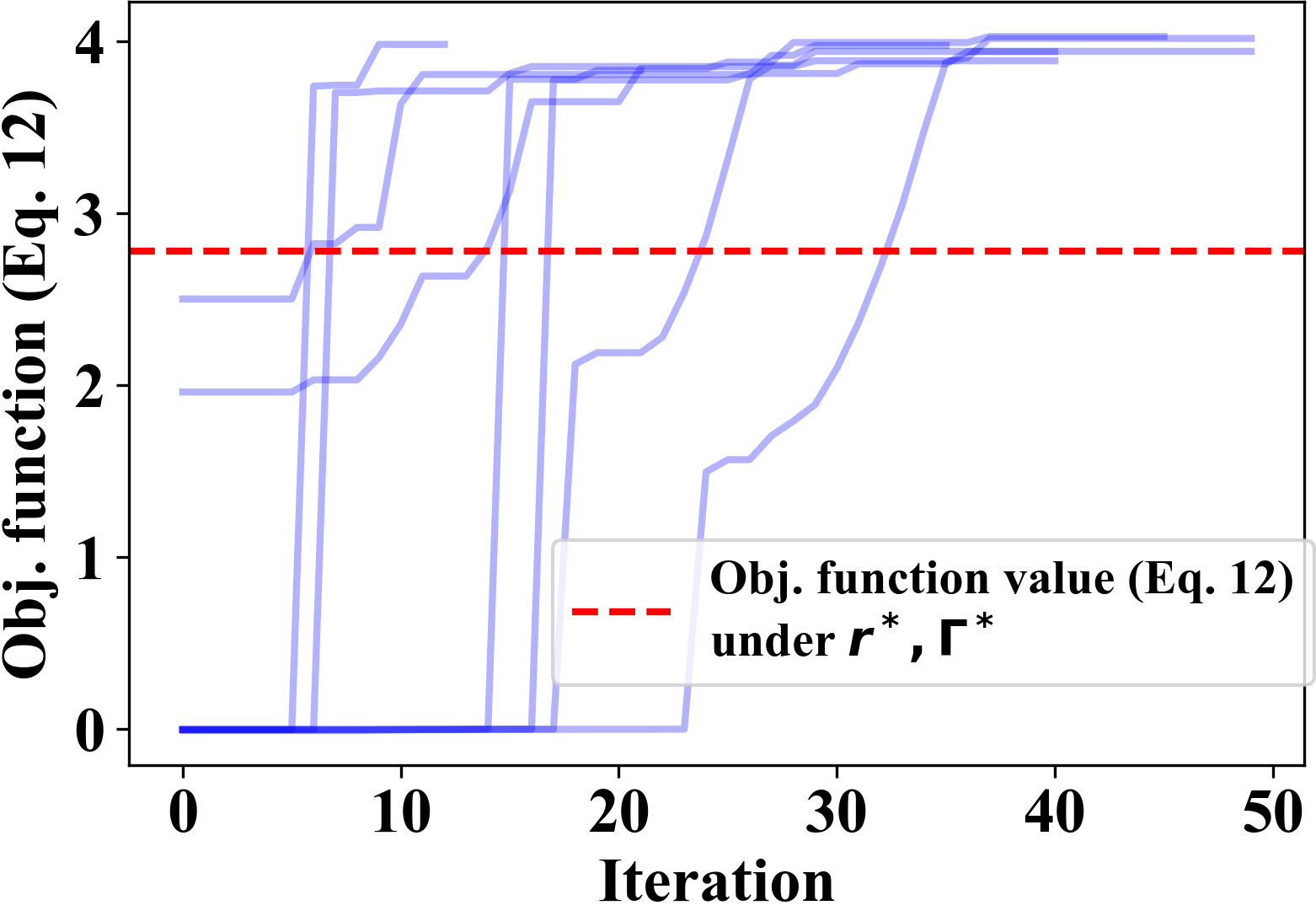}
    \caption{MPMCE-IRL:  the big-small domain}
    \label{apdx_fig:big_small_mce_bo_trace}
\end{subfigure}
~
\begin{subfigure}[t]{0.31\linewidth}
    \includegraphics[width=\textwidth]{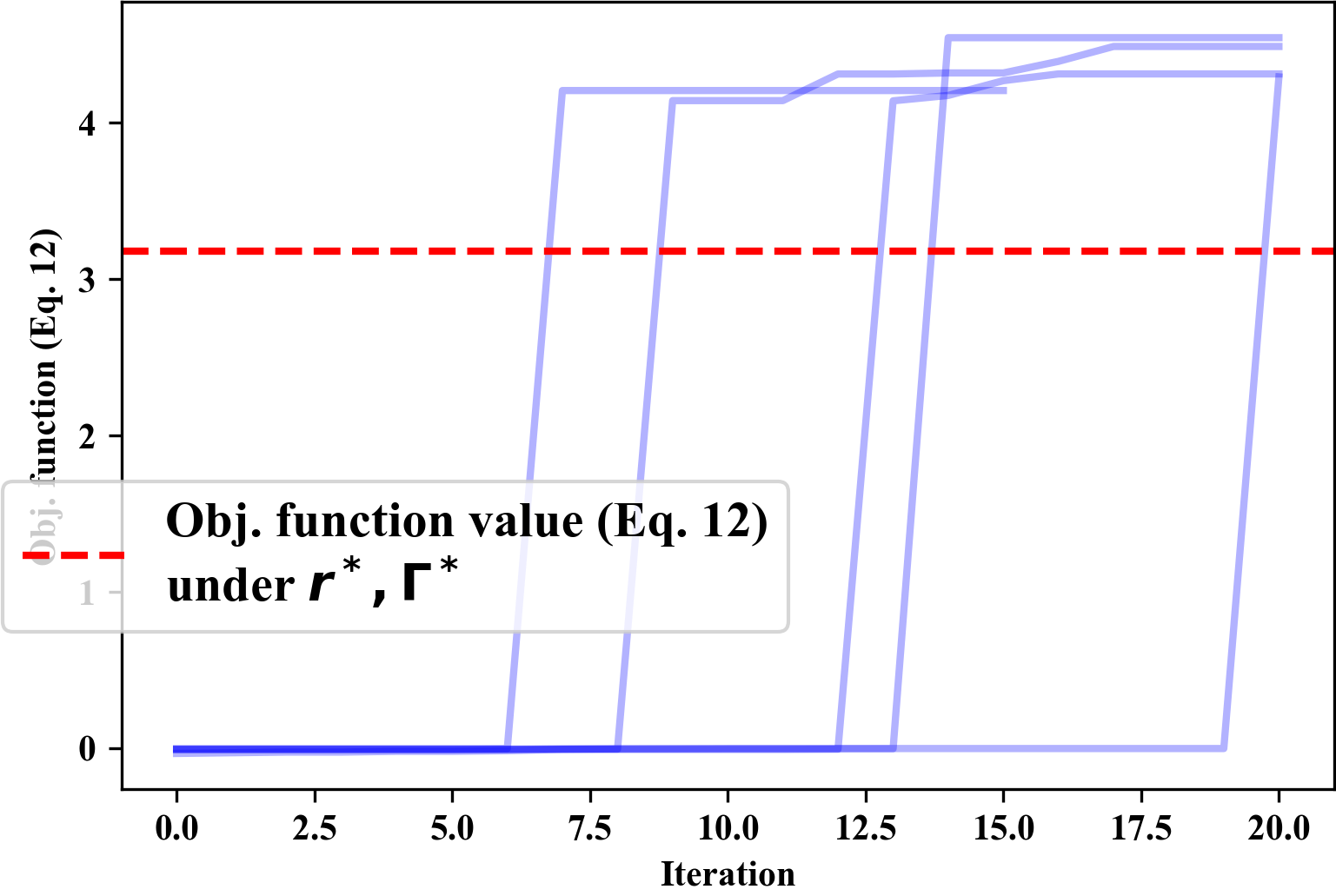}
    \caption{MPMCE-IRL: the cliff domain}
    \label{apdx_fig:cliff_mce_bo_trace}
\end{subfigure}
\caption{Trace plots of the best observed objective value of BO: $x,\ y$-axis represent the iteration and the best observed objective value, respectively. The red dashed line represents an approximate
global maximum from performing a grid search over $[0,1]^K$ or the objective value under the ground truth.}
\label{apdx_fig:bo_trace}
\end{figure}

\subsection{Identifiability and Generalizability Analysis for the Toy Domain}
In this section, we study the identifiability of the reward function for the toy domain, when the discount factors are misspecified. We only provide expert demonstrations from $2$ expert policies. We obtain $\Gamma$ on a grid space of $\Gamma\in[0,1]^K$ with an interval of $0.01$. For each set of assigned discount factors, 
we calculate $rank(\Phi|b)$ and $rank(\Phi)$. We find out that for all the given discount factors, $rank(\Phi|b)=rank(\Phi)=3|\mathcal{S}|-1$. By Proposition~\ref{prop:id}, this result implies that there exists reward functions that reconstruct expert policies for a large set of discount factors. 

We further study the generalizability of these reward functions. In Fig.~\ref{fig:id_and_gen}, we see that $\sim 40\%$ of these feasible reward functions do not generalize well to new tasks with generalization errors (Eq.~\ref{eqn:eval_gen}) larger than $0.5$, which emphasizes the importance of learning the discount factors correctly.

\begin{figure}[htbp]
\centering
\begin{subfigure}[t]{0.47\linewidth}
    \includegraphics[scale=0.5]{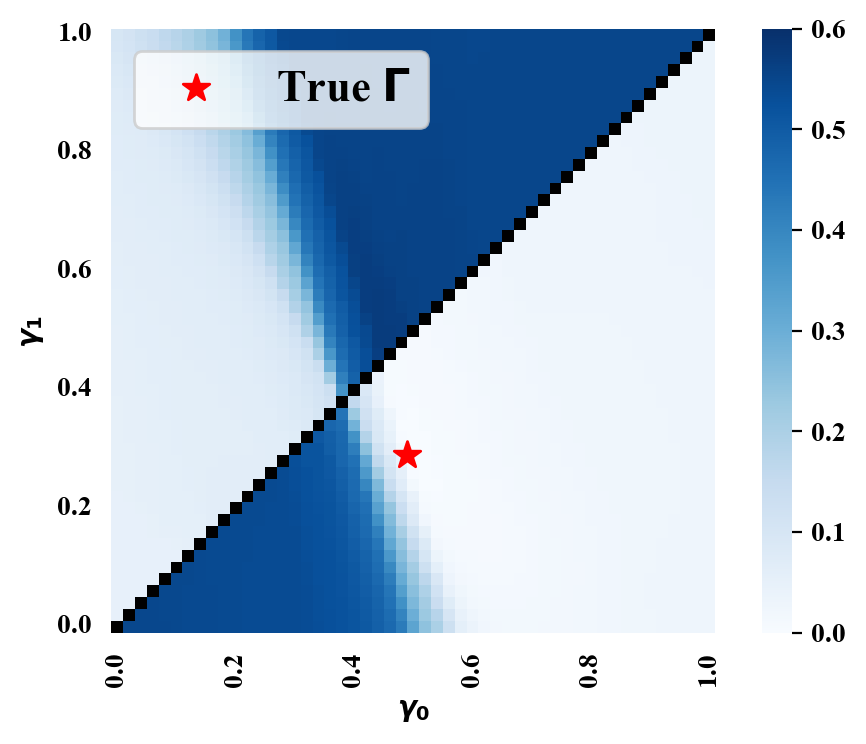}
    \caption{Heatmap of the generalization error of the reward functions (Eq.~\ref{eqn:eval_gen}) given (mis)specified discount factors.}
    \label{fig:gen}
\end{subfigure}
~
\begin{subfigure}[t]{0.47\linewidth}
    \includegraphics[width=\textwidth]{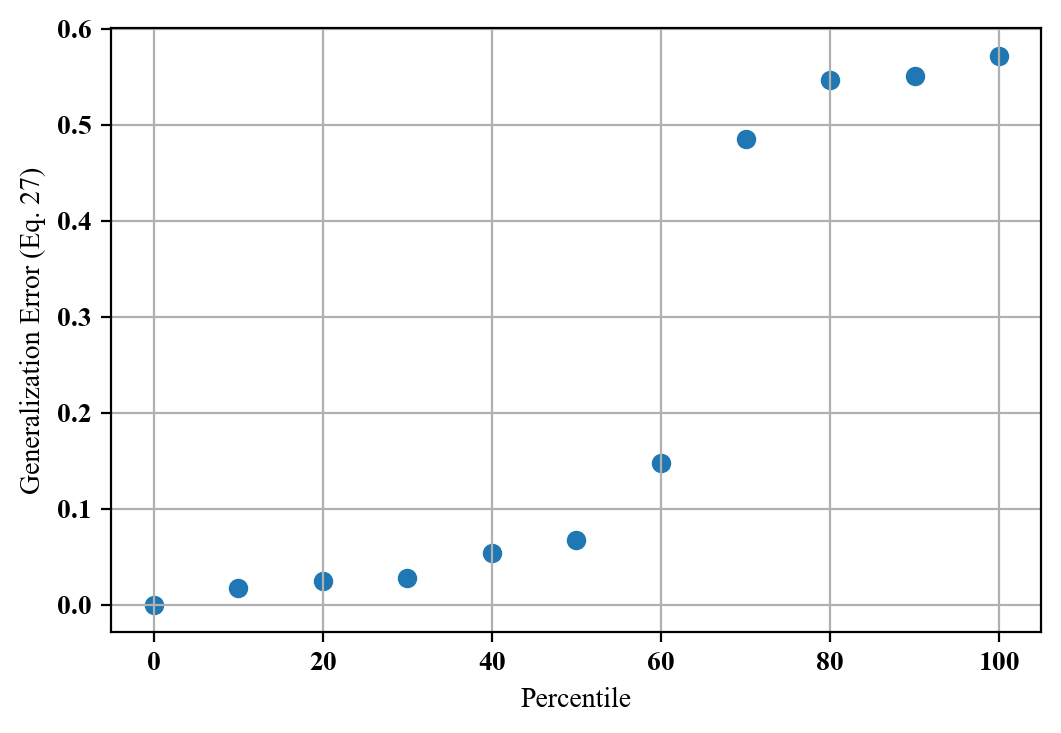}
    \caption{Percentile plot of the generalization error of the reward functions (Eq.~\ref{eqn:eval_gen}) given (mis)specified discount factors.)}
    \label{fig:id}
\end{subfigure}
\caption{The (a) heatmap and (2) percentile plot of the generalization error of the reward functions (Eq.~\ref{eqn:eval_gen}) given (mis)specified discount factors: In (a) $x,\ y$-axis represents the given $\gamma_0,\gamma_1$ respectively. 
The red star represents the true discount factors $\gamma^*_0,\gamma^*_1$. There are no feasible solutions along the diagonal. Lighter blue indicates a smaller error and vice versa. In (b) $x,\ y$-axis represents the percentile and the generalization error, respectively. }
\label{fig:id_and_gen}
\end{figure}

\end{document}